\documentclass[11pt]{article}
\usepackage[a4paper,margin=2.5cm]{geometry}
\usepackage{amssymb}
\usepackage{latexsym}
\usepackage{amsbsy}
\usepackage{footmisc}
\usepackage{amsfonts}
\usepackage{amsthm}
\usepackage{hyperref}
\usepackage{graphicx}
\usepackage{subcaption}
\usepackage{enumerate}
\usepackage{enumitem}
\usepackage{mathtools}
\usepackage{algorithm}
\usepackage{algpseudocode}
\usepackage{color}
\renewcommand{\thealgorithm}{\Alph{algorithm}}
\usepackage{booktabs}
\usepackage{multirow}
\usepackage{arydshln}
\usepackage{tikz}
\usepackage{float}

\newtheorem{theorem}{Theorem}[section]
\newtheorem{lemma}[theorem]{Lemma}
\newtheorem{proposition}[theorem]{Proposition}

\newtheorem{remark}[theorem]{Remark}
\newtheorem{assump}[theorem]{Assumption}

\numberwithin{equation}{section}

\title{ Conditional Diffusion Guidance under Hard Constraint: A Stochastic Analysis Approach}
\author{Zhengyi Guo, Wenpin Tang, and Renyuan Xu}
\author{Zhengyi Guo \thanks{Department of Industrial Engineering and Operations Research, Columbia University. \textbf{Email:}  zg2525@columbia.edu and wt2319@columbia.edu.} \and Wenpin Tang  \footnotemark[1] \and   Renyuan Xu \thanks{Department of Management Science and Engineering, Stanford University.  \textbf{Email:} renyuanxu@stanford.edu} }
\date{First version: February 2026; This version: June 2026.} 

\begin{document}
\maketitle

\begin{abstract}
We study conditional generation in diffusion models under hard constraints, where generated samples must satisfy prescribed events with probability one. Such constraints arise naturally in safety-critical applications and in rare-event simulation, where soft or reward-based guidance methods offer no guarantee of constraint satisfaction. Building on a probabilistic interpretation of diffusion models, we develop a principled conditional diffusion guidance framework based on Doob’s $h$-transform, martingale representation and quadratic variation process. Specifically, the resulting guided dynamics augment a pretrained diffusion with an explicit drift correction involving the logarithmic gradient of a conditioning function, 
without modifying the pretrained score network. Leveraging martingale and quadratic-variation identities, we propose two novel off-policy learning algorithms based on a martingale loss and a martingale–covariation loss to estimate $h$ and its gradient using only trajectories from the pretrained model. We provide non-asymptotic guarantees for the resulting conditional sampler in both total variation and Wasserstein distances, explicitly characterizing the impact of score approximation and guidance estimation errors. Numerical experiments demonstrate the effectiveness of the proposed methods in enforcing hard constraints and generating rare-event samples. 
The code of the numerical experiments can be found at \url{https://github.com/ZhengyiGuo2002/CDG_Finance}.
\end{abstract}
\section{Introduction}
\label{sc1}

Diffusion models have emerged as a powerful class of generative models capable of producing high-quality samples across a wide range of domains, including image synthesis, molecular design, and time series generation. 
Notably, they have played a central role in the success of text-to-image generators 
such as DALL·E 2 \cite{Ramesh22} and Stable Diffusion \cite{Rombach2022};
in text-to-video generators
such as Sora \cite{Sora2024}, Make-A-Video \cite{Singer2022} and Veo \cite{Veo2024};
and more recently, in diffusion large language models 
such as Mercury \cite{KK25} and LLaDA \cite{Nie25}.
In many downstream applications, however, the objective goes beyond unconditional sampling from a data distribution. Instead, we are often interested in generating samples that satisfy prescribed structural, functional, or feasibility constraints. Such requirements naturally give rise to the problem of guided or fine-tuned diffusion sampling, in which a generative model is adapted to target a conditional distribution associated with a desired event or property.

This demand for conditional generation arises from two closely related but distinct considerations.
First, in many applications, generated data must respect hard constraints dictated by physical laws, operational rules, or feasibility requirements. Examples include safety-critical systems, regulated decision-making pipelines, and constrained design problems, where violations are not allowed and must be ruled out at the level of the generated distribution.
Second, in domains such as finance, healthcare, and large-scale service systems, there is a growing need for stress testing and scenario analysis. In these settings, we seek to simulate rare but high-impact events (such as extreme market movements, adverse clinical outcomes, or system overloads) that occur with very low probability under the empirical data distribution. Addressing such problems requires generative models that can reliably produce samples from rare-event regimes rather than typical scenarios.

Modern diffusion models are typically trained in a pretraining phase to approximate the unconditional data distribution. Concretely, a pretrained score-based diffusion model specifies a stochastic sampling dynamics of the form:
\begin{equation*}
d Y_t = \Big(\bar f(t,Y_t)+ s_\theta(t, Y_t)\Big) dt + \overline{g}(t) dB_t, 
\end{equation*}
where $\{B_t\}_{t \ge 0}$ is standard Brownian motion, 
$\bar f(t,y)$ and $\bar g(t)$ are prescribed model parameters,
and 
$s_\theta(t,y)$ is a learned score function parametrized by $\theta$. The function $s_\theta$ is trained offline, using samples from the data distribution, to approximate the score (the gradient of the log-density) of the forward diffusion process (see Section \ref{sc2} for details).
For a well-trained diffusion model, the terminal distribution of $Y_T$ is expected to be close to the target data distribution $p_{\tiny \rm data}(\cdot)$.

The pretraining of diffusion models is computationally intensive and generally requires access to large and representative datasets. As a result, such models are typically trained once to approximate the unconditional data distribution and subsequently reused across multiple downstream tasks using guidance or fine-tuning. Given a pretrained diffusion model, existing guidance mechanisms can be broadly categorized into several paradigms, including supervised fine-tuning with regularization (often referred to as soft guidance)  \cite{KDR24, Tang24, UZ24}, conditioning-based approaches \cite{Dh21, HS21}, reinforcement learning–based methods \cite{Fan23,  GZZ24,  ZCZ24,  ZZT24b}, and more recent preference-optimization frameworks \cite{WDR24, YT24}. While soft-constrained or reward-based guidance methods are often computationally convenient, they incorporate target criteria through penalty or reward terms in the optimization objective and, as a consequence, generally do not guarantee that the resulting generated distribution is supported on the constraint set or that the constraint is satisfied with probability one. As a result, such methods may generate samples that violate prescribed constraints, which is a critical limitation in applications requiring strict constraint adherence or reliable rare-event simulation.

   In contrast, hard conditioning aims to generate samples by replacing the original generative law with the conditional law given a prescribed constraint set, under which the terminal output satisfies the constraint with probability one. This setting is substantially more challenging, particularly when the conditioning event is rare under the data distribution. In such cases, the conditional law concentrates on low-probability or geometrically thin regions of the data domain, rendering naive rejection sampling or importance reweighting ineffective. Enforcing hard constraints therefore requires modifying the entire generative trajectory, rather than only its terminal distribution. Designing a principled, theoretically grounded, and lightweight post-training mechanism that achieves this goal—without modifying the underlying pretrained score function—remains an {\it open problem} in the diffusion model literature.

  Compared to rejection sampling (see e.g., \cite{flury1990acceptance,gilks1992adaptive}), our hard-constrained diffusion guidance (introduced below) replaces post-hoc filtering with a principled change of measure that directly targets the conditional law. While rejection sampling produces exact conditional samples under the pretrained model, its acceptance probability equals the constraint probability~$\rho$, leading to an expected computational cost of order $\mathcal{O}(1/\rho)$ diffusion rollouts per feasible sample. This is computationally extremely expensive when the conditioning event is rare, and is also impractical for users. In contrast, our framework modifies the sampling dynamics via an additional drift term; once trained, conditional samples are generated in a single rollout, thereby avoiding this intrinsic inefficiency.

\paragraph{Our work and contribution.}
This paper develops a principled framework for Conditional Diffusion Guidance that directly addresses these challenges. Our approach is rooted in classical stochastic analysis and exploits the probabilistic structure underlying diffusion models. Specifically, we adopt a conditioning-based perspective closely related to Doob’s  $h$-transform, which characterizes the law of a diffusion process conditioned on a terminal constraint via a change of measure. In this formulation, the central object is the function:
\begin{eqnarray*}
    h(t,y) = \mathbb{P}(Y_T\in S|Y_t=y),
\end{eqnarray*}
where $S$  denotes the constraint set. The function 
$h$ represents the conditional probability under the pretrained diffusion dynamics, 
and induces a drift correction that enforces the constraint. Namely, the resulting guided dynamics take the form:
\begin{equation*}
dY^S_t = \Big( \bar f(t,Y^S_t) + s_\theta(t, Y^S_t) + \overline{g}(t)^2 \nabla \log h(t, Y^S_t) \Big)dt + \overline{g}(t) d\overline B_t.
\end{equation*}
where $\{\overline B_t\}_{t \ge 0}$ is a copy of Brownian motion. We provide theoretical justification for this construction by proving that the terminal output satisfies
\begin{eqnarray*}
    Y^S_T\sim (Z \mid Z \in S) \,\,\text{ with }\,\, Z \sim p_{\mathrm{data}}(\cdot).
\end{eqnarray*}
Importantly, the proposed guidance mechanism does not modify the pretrained score network $s_\theta$; instead, it augments the sampling dynamics with an additive guidance term $\nabla \log h$, making the method lightweight in both training and implementation.

Although Doob’s transform has been explored in the literature for guidance, existing approaches typically rely on stochastic control or reinforcement-learning formulations to enforce soft constraints. In contrast, our viewpoint is fundamentally probability-theoretic: we exploit martingale properties as well as quadratic-variation identities to derive learning objectives for the 
$h$-function and its gradient directly from the dynamics of the pretrained model. This leads to a powerful learning framework that is effective in enforcing hard constraints.

Specifically, leveraging the {\it (local) martingale property} of $h(t,Y_t)$, we propose the Conditional Diffusion Guidance via Martingale Loss (CDG-ML) algorithm, which learns $h(\cdot,\cdot)$ by minimizing the 
$L_2$ loss
\begin{eqnarray*}
   \min_{\ell(\cdot, \cdot)} \mathbb{E}_{[0,T]}\left[ \int_0^T \Big(\ell(t,Y_t) - 1(Y_T \in S)\Big)^2 dt\right].
\end{eqnarray*}
This is because $h(\cdot,\cdot)$ is the unique minimizer of the above objective function (under suitable conditions) and $\mathbb{E}_{[0,T]}[\cdot]$ is the expectation with respect to the law of the process $\{Y_t\}_{0\leq t\leq T}$. 

However, learning a good approximation to $h$ alone does not guarantee a good approximation of $\nabla \log h = \nabla h/h$, a difficulty that is well documented in practice. To address this challenge, we introduce a novel approach that learns 
$\nabla h$ directly via quadratic variation. Observing the quadratic covariation process follows $d[h, Y]_t = \overline{g}(t)^2 \nabla h(t, Y_t) dt$, we propose the  Conditional Diffusion Guidance via Martingale–Covariation Loss (CDG-MCL):
\begin{equation*}
\min_{q(\cdot,\cdot)} \mathbb{E}_{[0,T]}\left[\int_0^T\left|\frac{1}{\overline{g}(t)^2}\frac{d[h, Y]_t}{dt} - q(t, Y_t)\right|^2dt\right],
\end{equation*}
which estimates $\nabla h$ via quadratic-variation information. 

Both CDG-ML and CDG-MCL operate solely on samples from the pretrained diffusion, require no access to the underlying data distribution, and are grounded in tools from stochastic analysis. To the best of our knowledge, this combination of martingale and covariation-based learning objectives for diffusion guidance is novel in the literature. In practice, we parameterize $\ell(\cdot,\cdot)$ and $q(\cdot,\cdot)$ in both loss functions using neural networks.


\smallskip

A central contribution of this work is a rigorous theoretical analysis of the proposed framework. We quantify the discrepancy between the target conditional data distribution and the distribution induced by the learned guided dynamics using two complementary metrics: total variation distance and Wasserstein distance. In particular, we establish non-asymptotic total variation bounds that decompose the error into contributions from (i) the pretrained model approximation and (ii) the learned guidance error (Lemmas \ref{lem:predata} and \ref{lem:prepre}), leading to an explicit end-to-end guarantee for conditional sampling (Theorem \ref{thm:TV}). Total variation provides strong distributional guarantees under minimal structural assumptions. In contrast, under additional regularity and stability conditions, we derive Wasserstein-2 error bounds for the guided dynamics (Theorem \ref{thm:W2}), which offer a geometrically meaningful notion of error closely tied to contraction and stability properties of stochastic differential equations. Together, these results elucidate the trade-offs between statistical strength and analytical tractability in conditional diffusion guidance.

To complement the theoretical analysis, we present numerical experiments that investigate the effectiveness of the proposed Conditional Diffusion Guidance framework in enforcing hard constraints and generating rare-event scenarios. The experiments illustrate how guidance learned from pretrained diffusion trajectories reshapes the sampling dynamics to concentrate on low-probability regions of the data distribution that are inaccessible to naive sampling. We compare the CDG-ML and CDG-MCL algorithms in terms of constraint satisfaction, sampling stability, and distributional accuracy, thereby highlighting the practical implications of learning the guidance function versus its gradient. These results provide empirical support for the theoretical guarantees and demonstrate the applicability of the proposed framework to stress testing and rare-event simulation tasks.

\paragraph{Literature review.} In this paper, we focus on supervised guidance via endogenous conditioning. The underlying idea is closely related to Doob’s 
$h$-transform, with the central task being the estimation of the conditioning function 
$h$. Related works based on the 
$h$-transform framework, such as \cite{DEFT,DP24,pidstrigach2025conditioning,howard2025control}, formulate the problem from a stochastic control perspective and treat $\frac{\nabla h}{h}$ as the optimal control to be learned. In contrast, our approach is fundamentally probability-theoretic, relying on martingale properties and quadratic-variation identities of the pretrained diffusion. 
A key limitation of control-based conditional diffusion methods is that they are {\it inherently on-policy}: learning the guidance requires simulating trajectories under the evolving post-training dynamics. As a consequence, the data distribution used for learning depends on the current control approximation, introducing distribution shift and feedback effects that complicate stability analysis and the separation of modeling error from guidance error. Our framework avoids this issue by learning the conditioning function entirely off-policy using trajectories from the fixed pretrained diffusion, thereby decoupling learning from sampling.  In contrast to our work, none of these control-based papers provides approximation guarantees for their proposed conditional sampling methods.

Our work is related to the literature on classifier guidance  \cite{BCS23,  shenoy2024},  in which an auxiliary classifier/label is used to steer diffusion sampling toward desired attributes or outcomes.  This line of work includes both algorithmic developments \cite{BCS23, ND22} and recent efforts to understand the statistical properties of classifier-guided diffusion \cite{YHN24}. Classifier guidance also provides an important conceptual link to reinforcement learning from human feedback (RLHF), where human preference or reward models play a role analogous to classifiers by shaping the sampling dynamics of generative models \cite{Fan23,  GZZ24, han2025stochastic, HS21, ZCZ24,  ZZT24b, WC24}. For broader overviews and surveys of fine-tuning diffusion models using RLHF, we refer to \cite{CM24,UZ24tut}.
Conditioning also emerges in reinforcement learning 
as
goal-conditioned reinforcement learning \cite{Liu22, Sch15, Wang23},
which is primarily used for hard-constrained robotic task planning.


Our framework is also connected to pluralistic alignment \cite{SMF24,liu2024alignment}\footnote{Pluralistic alignment refers to multi-objective alignment aimed at integrating complex and potentially conflicting real-world values; see \url{https://pluralistic-alignment.github.io/}
 for recent developments.}, in which diffusion guidance is driven by multidimensional reward or preference structures (see \eqref{eq:general_conditioning} and the discussion thereafter). From a learning perspective, this connection places our work in close relation to multi-task learning, where multiple objectives or constraints are incorporated within a unified modeling framework.

Finally, the targeted applications are closely related to the literature on rare-event simulation \cite{SG07, Buck04}, where the objective is to efficiently sample from low-probability regimes, often beyond the reach of classical importance sampling techniques \cite{BL11, BL12}. There is also growing interest in the use of diffusion models for problems in operations research and stochastic simulation; see, for example, \cite{LZ24}.

\paragraph{Organization of the paper.} The remainder of the paper is organized as follows.
We start with background on diffusion models in Section \ref{sc2}.
In Section \ref{sc3}, 
we build the foundations for conditional diffusion guidance, leading to novel methodologies. 
In Section \ref{sc4}, we provide theoretical results of the proposed methodologies.
Extensions of the proposed methods are discussed in Section \ref{sc4+}.
Numerical experiments are reported in Section \ref{sc5}.

\paragraph{Notations.} Below we collect the notation used throughout.
\begin{itemize}[itemsep = 3 pt]
\item
$\mathbb{R}$ is the set of real numbers.
\item
For $x,y \in \mathbb{R}^d$, $x \cdot y$ denotes the scalar product of $x$ and $y$,
and $|x|: = \sqrt{x \cdot x}$ is the Euclidean norm of $x$.
\item
For  $A = (a_{ij})_{1 \le i, j \le d}$ a matrix, $|A|_F: = \sqrt{\sum_{i,j =1}^d a_{ij}^2}$ denotes  its Frobenius norm.
\item
For $f$ a function on $X$, $|f|_\infty:= \sup_{X}|f(x)|$ denotes its sup-norm.
\item
For $f:[0,\infty) \times \mathbb{R}^d \ni (t,x) \to \mathbb{R}$,
$\partial_t f$ denotes its time derivative,
$\nabla f$ is the gradient of $f$
and $\partial_k f:= \frac{\partial f}{\partial x_k}$ its $k^{th}$ coordinate,
and $\Delta f:= \sum_{k=1}^d \frac{\partial^2f}{\partial x_k^2}$ is the Laplacian of $f$.
\item
For $Z$ a random variable, $\mathbb{E}Z$ denotes the expectation of $Z$.
\item
Let \(p(\cdot)\) and \(q(\cdot)\) be two probability distributions defined on the same measurable space.
The total variation distance between \(p\) and \(q\) is defined by $d_{\mathrm{TV}}(p,q) := \sup_{A} \bigl| p(A) - q(A) \bigr|.$
The Kullback--Leibler (KL) divergence from \(q\) to \(p\) is given by $d_{\mathrm{KL}}(p,q) := \int \log\!\left(\frac{dp}{dq}\right)\, dp,$ whenever \(p\) is absolutely continuous with respect to \(q\). Finally, the Wasserstein--\(2\) distance between \(p\) and \(q\) is defined as $W_2(p,q) := \left( \inf_{\gamma} \mathbb{E}_{(X,Y)\sim\gamma}
\bigl[ |X-Y|^2 \bigr] \right)^{1/2},$
where the infimum is taken over all couplings \(\gamma\) of \(p\) and \(q\).
\end{itemize}
We use $C$ for a generic constant whose values may change from line to line.

\section{Background on diffusion models}
\label{sc2}

   This section provides preliminaries of score-based diffusion models. 
We follow the presentation of \cite{TZ24tut}.
Diffusion models rely on a forward-backward procedure:
the forward process transforms the target data to noise,
and the backward process recovers the data from noise. 

   Let $p_{\tiny \rm{data}}(\cdot)\in\mathcal{P}(\in\mathbb{R}^d)$ be the target data distribution.
Fixing $T > 0$, the forward process $\{X_t\}_{0 \le t \le T}$ is governed by 
the stochastic differential equation (SDE):
\begin{equation}
\label{eq:forward}
dX_t = f(t, X_t) dt + g(t) dW_t, X_0 \sim p_{\tiny \rm{data}}(\cdot),
\end{equation}
where $f: \mathbb{R}_+ \times \mathbb{R}^d \to \mathbb{R}^d$, $g: \mathbb{R}_+ \to \mathbb{R}_+$,
and $\{W_t\}_{t \ge 0}$ is Brownian motion in $\mathbb{R}^d$. 
We require some usual measurable conditions on $f(\cdot, \cdot)$ and $g(\cdot)$ so that the SDE \eqref{eq:forward} is well-defined, 
and that $X_t$ has a smooth probability density $p(t, x):= \mathbb{P}(X_t \in dx)/dx$
(see \cite{SV79}).
By the time reversal formula \cite{Ander82, HP86}, let
\begin{equation}
\label{eq:truetr}
d\overline{X}_t = \overline{f}^\circ(t,\overline{X}_t) dt + g(T-t) dB_t, 
   \overline{X}_0 \sim p(T, \cdot),
\end{equation}
whose drift is:
\begin{equation}\label{eq:true_f}
\overline{f}^\circ(t,y) = -f(T-t, y) + \overline{g}(t)^2 \nabla \log p(T-t, y).
\end{equation}
Here $\{B_t\}_{t \ge 0}$ is a copy of Brownian motion in $\mathbb{R}^d$.
The processes $\{\overline{X}_t\}_{0 \le t \le T}$ and $\{X_{T-t}\}_{0 \le t \le T}$ have the same marginal distribution \cite{Ander82, HP86}. 
Hence the output $\overline{X}_T$ follows the desirable distribution $p_{\tiny \rm{data}}(\cdot)$. 

\begin{remark}[Examples]
    In practice, notable examples include the {\em Variance Exploding} (VE) model \cite{Karras22} and 
{\em Variance Preserving} (VP) model \cite{Song20}, which serves as a demonstrative case throughout the paper:
\begin{itemize}
    \item (VE): $f(t,x) = 0$, $g(t) = \sqrt{2t+\epsilon}$ for some (small) $\epsilon>0$.
    \item (VP): $f(t,x) = -\frac{1}{2}\left(a+ \frac{(b-a)t}{T}\right)x$, $g(t) = \sqrt{a+ \frac{(b-a)t}{T}}$ for some $b > a > 0$. 
\end{itemize}
\end{remark}

    The main challenge in sampling the process $\overline{X}$ is that the data distribution $p_{\tiny\mathrm{data}}(\cdot)$ is unknown, and consequently the associated score function $\nabla \log p(\cdot,\cdot)$ is not available.
Moreover, the initialization $\overline{X}_0 \sim p(T,\cdot)$ required for each sample generation also depends on $p_{\tiny\mathrm{data}}(\cdot)$.
The idea of score-based diffusion models is to learn the score function via a parametrized family of functions $\{s_\theta(t,x)\}_\theta$ (e.g., neural networks), with a limited number of samples from $p_{\tiny \rm{data}}(\cdot)$ (see \cite{HJA20, song2019generative, Song20}).
The resulting backward process $\{Y_t\}_{0 \le t \le T}$ for sampling is:
\begin{equation}
\label{eq:back}
d Y_t = \overline{f}(t, Y_t) dt + \overline{g}(t) dB_t,    Y_0 \sim p_{\tiny \rm{noise}}(\cdot),
\end{equation}
where $\overline{g}(t):= g(T-t)$ and 
\begin{eqnarray}\label{eq:approx_f}
    \overline{f}(t,y):= -f(T-t, y) + g(T-t)^2 s_{\theta_*}(T-t, y),
\end{eqnarray}
serves as an approximation to $\overline{f}^\circ$, which is defined in \eqref{eq:true_f}.
Here,
\begin{itemize}[itemsep = 3 pt]
\item
$p_{\tiny \rm{noise}}(\cdot)$ 
is a proxy to $p(T, \cdot)$ for generating the target distribution from noise, which should {\em not} depend on $p_{\tiny \rm{data}}(\cdot)$. The form of $p_{\tiny \rm{noise}}(\cdot)$ is related to the design of the diffusion model, 
i.e., the pair $(f(\cdot, \cdot), g(\cdot))$.
In practice, we usually take $p_{\tiny \rm{noise}}(\cdot)$ as $  \mathcal{N}(0, T^2 I)$ for the VE model and  $p_{\tiny \rm{noise}}(\cdot) $ as $ \mathcal{N}(0, I)$ for the VP model.
 \item
$\{s_\theta(t,x)\}_\theta$ are function approximations to the score $\nabla \log p(t,x)$,
which are trained by solving some stochastic optimization schemes.
This technique is known as {\em score matching},
and the learned $s_{\theta_*}(t,x)$ is the {\em score matching function}.
There are several existing score matching methods, 
among which the most widely used one is the {\em denoising score matching} (DSM) \cite{Hyv05, Vi11}:
\begin{equation*}
\qquad \min_\theta \int_0^T \lambda(t) \, \mathbb{E}_{X_0\sim p_{\rm \scriptsize data}(\cdot)} \left[ \mathbb{E}_{X_t | X_0}\Big|s_{\theta}(t,X_t)- \nabla \log p(t,X_t | X_0)\Big|^2 \right] dt,
\end{equation*}
where $\lambda: \mathbb{R}_+ \to \mathbb{R}_+$ is a weight function.
\end{itemize}

   Most successful diffusion models rely on extensive score-matching training on large, generic datasets and are therefore trained once in advance and reused across downstream tasks; such models are commonly referred to as {\it pretrained models}.
\footnote{It was shown in \cite{Karras22} that a wide class of diffusion models can be obtained from the VE model via reparameterization. Consequently, in practice it suffices to pretrain a high-quality VE model, that is, to learn its score function with sufficient accuracy.}.
Throughout, we let the backward process \eqref{eq:back} represent such a pretrained model.
Denote by $P_{[0,T]}(\cdot)$ the law of the process $\{Y_t\}_{0 \le t \le T}$ on path space under which $\mathbb{E}_{[0,T]}[\cdot]$ is the associated expectation, and by $P_t(\cdot)$ its marginal distribution at time $t$ with $\mathbb{E}_{t}[\cdot]$ the associated expectation.
We write
\begin{equation*}
p_{\tiny \rm{pre}}(\cdot):= P_T(\cdot),
\end{equation*}
to denote the output distribution of the pretrained model.
A well-trained model is expected to generate reliable samples in the sense that
$p_{\tiny \rm{pre}}(\cdot)$ provides a good approximation of the data distribution
$p_{\tiny \rm{data}}(\cdot)$.
Under suitable conditions on the model components
${f(\cdot,\cdot), g(\cdot), T, p_{\tiny \rm{noise}}(\cdot), s_\theta(\cdot,\cdot)}$
and on the target distribution $p_{\tiny \rm{data}}(\cdot)$,
this approximation can be made quantitative by bounding appropriate discrepancies between
$p_{\tiny \rm{pre}}(\cdot)$ and $p_{\tiny \rm{data}}(\cdot)$. The following proposition bounds the total variation and the Wasserstein distance between 
$p_{\tiny \rm{pre}}(\cdot)$ and $p_{\tiny \rm{data}}(\cdot)$ 
for the aforementioned VE and VP models. 
\begin{proposition}
\label{prop:VEVP}
Assume that $\mathbb{E}_{p_{\tiny \rm{data}}(\cdot)}|X|^2 < \infty$, 
and that the score matching satisfies:
\begin{equation}
\label{eq:sme}
\sup_{0 \le t \le T}\mathbb{E}_{p(t, \cdot)} |s_{\theta_*}(t, X) - \nabla \log p(t, X)|^2 \le \varepsilon^2_{\tiny \texttt{VP}}    \quad (\mbox{resp}.\,\, \varepsilon^2_{\tiny \texttt{VE}}),
 \end{equation}
 for the VE (resp. VP) model. Then the following results hold.
\begin{enumerate}[itemsep = 3 pt]
\item (Total variation) \cite{Chen23, TZ24tut}
There are $C_{\tiny \mbox{VE}}, C_{\tiny \mbox{VP}}> 0$ (independent of $T$) such that
\begin{equation}
d_{TV}\left(p_{\tiny \rm{pre}}(\cdot), p_{\tiny \rm{data}}(\cdot) \right) \le 
\left\{ \begin{array}{lcl}
C_{\tiny \texttt{VE}} \left( T^{-1}  \sqrt{\mathbb{E}_{p_{\tiny \rm{data}}(\cdot)}|X|^2}+ \varepsilon_{\tiny \texttt{VE}} \sqrt{T} \right) & \mbox{for VE}, \\
C_{\tiny \texttt{VP}} \left(e^{-C_{\tiny \texttt{VP}}\,T} \sqrt{\mathbb{E}_{p_{\tiny \rm{data}}(\cdot)}|X|^2} + \varepsilon_{\tiny \texttt{VP}} \sqrt{T}\right) & \mbox{for VP}.
\end{array}\right.
\end{equation}
\item (Wasserstein) \cite{GNZ23, TZ24tut}
Assume further that $p_{\tiny \rm{data}}(\cdot)$ is $\kappa$-strongly log-concave 
\footnote{A smooth function $\ell: \mathbb{R}^d \to \mathbb{R}$ is $\kappa$-strongly log-concave if $\langle\nabla \log \ell(x) - \nabla \log \ell(y),  (x-y) \rangle \le -\kappa |x-y|^2$ for all $x,y$.} 
for $\kappa$ sufficiently large.
There are $C_{\tiny \mbox{VE}}, C_{\tiny \mbox{VP}}> 0$ (independent of $T$) such that
\begin{equation}
W_2\left(p_{\tiny \rm{pre}}(\cdot), p_{\tiny \rm{data}}(\cdot) \right) \le 
\left\{ \begin{array}{lcl}
C_{\tiny \texttt{VE}} \left( T^{-1}  \sqrt{\mathbb{E}_{p_{\tiny \rm{data}}(\cdot)}|X|^2} + \varepsilon_{\tiny \texttt{VE}}\,T^{2} \right) & \mbox{for VE}, \\
C_{\tiny \texttt{VP}} \left(e^{-C_{\tiny \texttt{VP}}\,T} \sqrt{\mathbb{E}_{p_{\tiny \rm{data}}(\cdot)}|X|^2} + \varepsilon_{\tiny \texttt{VP}} \right)  & \mbox{for VP}.
\end{array}\right.
\end{equation}
\end{enumerate}
\end{proposition}

   The condition \eqref{eq:sme} formalizes a black-box score-matching error, an assumption that is standard in the recent literature on the convergence analysis of diffusion models \cite{Chen23, GNZ23, LLT22, LW23, TZ24}. In this line of work, the statistical and algorithmic aspects of score estimation are treated as a separate problem, and the learned score function is taken as an approximate oracle whose error is summarized by \eqref{eq:sme}. Complementary to this approach, there exists a body of work \cite{HR24,  WHT24} that derives quantitative rates for score matching itself, typically under structural assumptions such as low-dimensional or parametric representations of the score function class $\{s_\theta(t,x)\}_\theta$. However, state-of-the-art score-based diffusion models in practice rely on highly expressive and deep neural network architectures \cite{Karras22, Song20}, for which such structural assumptions are difficult to justify. For this reason, and in line with the convergence-focused literature, we adopt the black-box score-matching assumption \eqref{eq:sme} throughout this paper.

\section{Diffusion guidance by conditioning}
\label{sc3}

In this section, we introduce a \emph{Conditional Diffusion Guidance} framework and  establish the mathematical foundations. Our focus is on enforcing \emph{hard constraints} on the generated samples, as opposed to encouraging desirable outcomes through soft reward signals.

Let $S \subset \mathbb{R}^d$ be a Borel set, which we refer to as the \emph{guidance set}, and denote by $p^S_{\mathrm{data}}(\cdot)$ the target distribution conditioned on $S$. That is, for a random variable $Z \sim p_{\mathrm{data}}(\cdot)$,
\[
(Z \mid Z \in S) \sim p^S_{\mathrm{data}}(\cdot).
\]
The set $S$ encodes the event or structural constraint of interest and may represent a wide range of conditions on the data, including discrete labels, functional constraints, rare events, or application-specific admissible regions. For example, in time-series generation, $S$ may characterize regimes exhibiting certain statistical or dynamical properties, such as high or low volatility periods, prescribed terminal values, or specific temporal patterns. Throughout, we assume that $S$ is sufficiently regular (e.g., Borel and non-negligible) so that the conditional distribution $p^S_{\mathrm{data}}(\cdot)$ is well-defined.

   From a modeling perspective, conditioning on $S$ imposes a \emph{hard constraint}, in the sense that generated samples are required to satisfy the event $\{Z \in S\}$ almost surely. This should be contrasted with the more commonly studied \emph{soft conditioning} or reward-based guidance approaches. While soft constraints are often easier to optimize, they generally do not guarantee constraint satisfaction and may yield samples that violate physical laws, feasibility requirements, or operational rules. Hard conditioning is substantially more challenging, particularly when the event $\{Z \in S\}$ is rare under the data distribution, causing the conditional law $p^S_{\mathrm{data}}(\cdot)$ to concentrate on low-probability or geometrically thin regions of the state space. In such settings, naive rejection sampling or reward-based reweighting becomes ineffective, and enforcing the constraint typically requires modifying the entire generative trajectory rather than only its terminal distribution. These considerations motivate a conditioning-based framework that directly targets the conditional law associated with the event $\{Z \in S\}$.

   Our framework can be easily extended to settings in which the guidance set $S$ is specified implicitly through a measurable functional $F: \mathbb{R}^d \to \mathcal{Y}$, so that the conditioning event takes the form
\begin{equation}\label{eq:general_conditioning}
    F(Y_T) \in S.
\end{equation}
This formulation covers constraints such as conditions on cumulative or averaged quantities, path-dependent risk measures, tail events of aggregate observables, and system-level performance metrics arising in operations research and related application domains. 
It also accommodates multi-objective guidance by taking $F=(F_1,\cdots,F_K):\mathbb{R}^d\rightarrow\mathbb{R}^K$ and imposing hard constraints $F_k(Y_T)\in S_k$ for each objective. Pluralistic alignment can be modeled by allowing multiple acceptable guidance sets $\{S^{(j)}\}_{j=1}^K\subset \mathbb{R}^K$, corresponding to distinct alignment criteria, and targeting a mixture of the resulting conditional distributions.

   The problem of interest is to exploit a given pretrained models to generate samples that approximate the conditional distribution $p^S_{\tiny \rm{data}}(\cdot)$. Recall that $p_{\tiny \rm{pre}}(\cdot)$ denotes the output distribution of the pretrained model~\eqref{eq:back}, and let $p^S_{\tiny \rm{pre}}(\cdot)$ be its conditioning on $S$. As discussed in Section~\ref{sc2}, a well-trained pretrained model yields $p_{\tiny \rm{pre}}(\cdot) \approx p_{\tiny \rm{data}}(\cdot)$, which in turn implies $p^S_{\tiny \rm{pre}}(\cdot) \approx p^S_{\tiny \rm{data}}(\cdot)$ under mild assumptions on the guidance set $S$ (see Section~\ref{sc4}). Therefore, our goal is to use the pretrained model to sample from $p^S_{\tiny \rm{pre}}(\cdot)$, or more precisely, to generate sample paths
\begin{equation}
\label{eq:condiff}
\{Y^S_t\}_{0 \le t \le T} \sim P_{[0,T]}(Y \mid Y_T \in S).
\end{equation}

\subsection{Conditional diffusion sampling}
\label{sc31}
We start by discussing how the process $\{Y^S_t\}_{0 \le t \le T}$
is generated. 
To make the framework meaningful, assume that $P_{[0,T]}(Y_T\in S)>0$.
For $t \ge 0$, let
\begin{equation}
\label{eq:hfunc}
h(t,y):= P_{[0,T]}(Y_T \in S \,|\, Y_t = y).
\end{equation}
By the Baye's rule, we have for $s > 0$,
\begin{equation}
\label{eq:Bayes}
\begin{aligned}
P_{[0,T]}(Y_{t+s} = y' \,|\, Y_t = y, \, Y_T \in S)
& = \frac{P_{[0,T]}(Y_{t+s}= y', Y_T \in S \,|\, Y_t = y) }{P_{[0,T]}(Y_T \in S \,|\, Y_t = y)} \\
& = \frac{h(t+s, y')}{h(t,y)}P_{[0,T]}(Y_{t+s} = y' \,|\, Y_t = y),
\end{aligned}
\end{equation}
which is known as {\em Doob's $h$-transform} \footnote{Doob's $h$-transform is a general concept of conditioning a Markov process, provided that $h$ is harmonic with respect to the Markov generator. In our setting, the $h$ function \eqref{eq:hfunc} arises as a special case by conditioning on the terminal data.
The ``bridge" calculation in \eqref{eq:Bayes} can be understood in terms of transition densities,
see \cite[Section 2]{FPY92} for a justification. \label{foot:h}}. Note that
$h(t,y):= P_{[0,T]}(Y_T \in S \,|\, Y_t = y)$ is harmonic with respect to the generator of 
$\{Y_t\}_{0 \le t \le T}$ and satisfies the following second-order PDE:
\begin{equation}\label{eq:h_pde}
\partial_t h + \overline{f}(t,y) \cdot \nabla h + \frac{1}{2} \overline{g}(t)^2 \Delta h = 0,
\end{equation}
with the terminal condition $ h(T, \cdot) = 1(\cdot \in S)$. 
It is well known that the solution to the linear parabolic equation \eqref{eq:h_pde}
is smooth in $[0,T) \times \mathbb{R}^d$
if $\overline{g}(\cdot)$ is bounded away from zero,
and $\overline{f}(\cdot, \cdot)$, $\overline{g}(\cdot)$ are H\"older continuous
(see \cite[Chapter 3, Theorem 5]{Fried64} or \cite[Theorem 6]{Ol62} for the interior a priori estimate).

The $h$-transform \eqref{eq:Bayes}
reveals a change of measure between the processes $\{Y_t\}_{0 \le t \le T}$ and $\{Y^S_t\}_{0 \le t \le T}$.
As a result, $\{Y_t^S\}_{0 \le t \le T}$ is also a diffusion process, 
whose dynamics is given in the following proposition.

\begin{proposition}\label{prop:Y_T^S}
The distribution of $\{Y^S_t\}_{0 \le t < T}$ is governed by:
\begin{equation}
\label{eq:backguide}
dY^S_t = \left(\overline{f}(t, Y^S_t) + \overline{g}(t)^2 \nabla \log h(t, Y^S_t) \right)dt + \overline{g}(t) d\overline B_t,   
\quad
Y_0^S \sim \nu_0^S(\cdot),
\end{equation}
where $\overline B_t$ is an $\mathcal{F}$-adapted Brownian motion; $\overline{f}(\cdot, \cdot)$ and $\overline{g}(\cdot)$ are given by the pretrained model \eqref{eq:back}, and the initial distribution is the conditional initial law
\[
    \nu_0^S(dy)
    :=
    \frac{h(0,y)}{P_{[0,T]}(Y_T\in S)}\,p_{\tiny \rm noise}(dy).
\]
Define $Y^S_T = \lim_{t\rightarrow T}Y^S_t$, then we have
\begin{equation}\label{eq:limit}
    Y_T^S \sim p_{\tiny\rm pre}^S(\cdot).
\end{equation}
\end{proposition}

The proof for \eqref{eq:backguide}  is a direct application of \eqref{eq:Bayes} and Girsanov's theorem (see \cite[IV.39]{RW00} and \cite[Theorem 2]{jamison1975markov}).  The key is to prove \eqref{eq:limit} using a limiting argument.
\begin{proof}
    By standard PDE estimates, we have for each $\varepsilon>0$, $h\in C^{1,2}([0,T-\varepsilon]\times \mathbb{R}^d)$ and $h(t,y)>0$ on $[0,T-\varepsilon]\times \mathbb{R}^d$. In addition, 
    \begin{equation}\label{eq:exp_bound}
        \mathbb{E}_P\left[\exp\Big( \frac{1}{2}\int_0^{T-\varepsilon}|\overline{g}(t)\nabla \log h(t,Y_t)|^2dt\Big)\right]<\infty.
    \end{equation}

Let $\Omega:=C([0,T];\mathbb{R}^d)$, $Y_t(\omega):=\omega(t)$ be the coordinate process and $\mathcal{F}_t=\sigma(Y_s:s\leq t)$ the canonical filtration.

First, let us define the normalized terminal indicator $Z:=\frac{1(Y_T\in S)}{P_{[0,T]}(Y_T\in S)}$. For $t\in[0,T]$, set $M_t = \mathbb{E}[Z|\mathcal{F}_t]$. Then
\begin{equation*}
    \mathbb{E}_P[Z|\mathcal{F}_t] = \frac{1}{P_{[0,T]}(Y_T\in S)}{P}_{[0,T]}(Y_T\in S|\mathcal{F}_t) = \frac{1}{P_{[0,T]}(Y_T\in S)}h(t,Y_t),
\end{equation*}
    using the Markov property. Moreover, since $(M_t)$ is a uniformly integrable martingale (bounded by $\frac{1}{P_{[0,T]}(Y_T\in S)}$), we have
    \begin{equation}\label{eq:M_conv}
        M_t \rightarrow M_T = Z \in L^1 \text{ and } a.s., \text{ as } t \rightarrow T.
    \end{equation}
    For each $\varepsilon>0$, define a probability measure on $\mathcal{F}_{T-\varepsilon}$ by
    \begin{equation*}
   {Q}^\varepsilon(A):=\mathbb{E}_P[1_AM_{T-\varepsilon}], \qquad A\in\mathcal{F}_{T-\varepsilon}.
    \end{equation*}
    If $0<\varepsilon'<\varepsilon$, then $Q^\varepsilon$ and $Q^{\varepsilon'}$ agree on $\mathcal{F}_{T-\varepsilon}$:
    \begin{eqnarray}
        Q^{\varepsilon'}(A) = \mathbb{E}_P[1_AM_{T-\varepsilon'}] = \mathbb{E}_P[1_A\mathbb{E}_P[M_{T-\varepsilon'}|\mathcal{F}_{T-\varepsilon}]]=   Q^{\varepsilon}(A),
    \end{eqnarray}
    since $\{M_t\}_{t \ge 0}$ is a martingale. By Carathéodory's extension theorem
 on the increasing sigma-fields, there is a unique probability measure $Q$ on $\mathcal{F}_{T-}:=\sigma(\cup_{t<T}\mathcal{F}_t)$ such that $Q|_{\mathcal{F}_{T-\varepsilon}}=Q^\varepsilon$ for all $\varepsilon>0$.

Let $A\in \mathcal{F}_{T-}$, by taking a sequence $A_n\in \mathcal{F}_{t_n}$ with $t_n\rightarrow T$ such that $1_{A_n}\rightarrow 1_A$ in $L^1$, we can show that
\begin{equation*}
    Q(A_n) = \mathbb{E}_P[1_{A_n}M_{t_n}] = \frac{{P}_{[0,T]}(A_n\cap \{Y_T\in S\})}{{P}_{[0,T]}(Y_T\in S)}.
\end{equation*}
    Let $n\rightarrow \infty$ and use dominated convergence on both sides (everything bounded) to obtain 
    \begin{eqnarray}\label{eq:Q}
        Q(A) = P_{[0,T]}(A|Y_T\in S).
    \end{eqnarray}

In particular, for any Borel set $A\subset \mathbb{R}^d$,
\begin{equation*}
    Q(Y_0\in A)
    =
    \mathbb{E}_P[1_{\{Y_0\in A\}}M_0]
    =
    \frac{1}{P_{[0,T]}(Y_T\in S)}
    \int_A h(0,y)\,p_{\tiny \rm noise}(dy).
\end{equation*}
Thus, under $Q$, the initial law of the coordinate process is $\nu_0^S$.

Define the conditional process $Y^S$ as the coordinate process under $Q$: on canonical path space $\Omega$ with coordinate map $Y_t(\omega)=\omega(t)$, set 
\begin{eqnarray}\label{eq:Y^S_coordinate}
    Y_t^S:=Y_t, \text{ as a random variable on } (\Omega, \mathcal{F}_{T-},Q), \qquad 0\leq t<T.
\end{eqnarray}

Fix $\varepsilon>0$ and for $t\in[0,T-\varepsilon]$, applying Itô's formula, \eqref{eq:h_pde} and \eqref{eq:M_conv} yields
\begin{equation*}
    d M_t = M_t \overline{g}(t)\nabla \log h(t,Y_t)\cdot d B_t.
\end{equation*}
 Thanks to \eqref{eq:exp_bound}, Girsanov can be applied on $[0,T-\varepsilon]$. In addition,   
 $B_t^Q := B_t - \int_0^t\overline{g}(s)\nabla \log h(s,Y_s)ds$ is a $Q$-Brownian motion for $t\leq T-\varepsilon$. Substituting into  \eqref{eq:back}, for $t\leq T-\varepsilon$,
 \begin{equation}\label{eq:Y_tB^Q}
     d Y_t = (\overline{f}(t,Y_t)+\overline{g}(t)^2 \nabla \log h(t,Y_t))dt + \overline{g}(t)d B_t^Q.
     \end{equation}
Using the identification \eqref{eq:Y^S_coordinate}, this is exactly the claimed dynamics for $Y^S$ on $[0,T-\varepsilon]$ in \eqref{eq:backguide}, with initial law $\nu_0^S$; and since $\varepsilon>0$ is arbitrary, it holds for all $t<T$. 

Finally, define the terminal value by $Y_T^S :=\lim_{t\rightarrow T}Y_t^S$ which exists 
$Q$-a.s. by path continuity. Moreover, for any Borel $A\in \mathbb{R}^d$, the event $\{Y_T^S\in A\}$ belongs to $\mathcal{F}_{T-}$ because $Y_T^S :=\lim_{n\rightarrow \infty}Y_{T-\frac{1}{n}}^S$.
Hence, applying \eqref{eq:Q} with $A$ replaced by $\{Y_T^S\in A\}$ yields
\begin{eqnarray*}
    Q(Y_T^S\in A) = P_{[0,T]}(Y_T\in A|Y_T\in S), \qquad A\in \mathcal{B}(\mathbb{R}^d). 
\end{eqnarray*}
This completes the proof.
\end{proof}

As noted in \cite[Appendix B]{UZ24}, the dynamics \eqref{eq:backguide} can be interpreted as classifier-guided diffusion sampling \cite{Dh21} in continuous time.
Let \(P^S_{[0,T]}(\cdot)\) denote the law of the guided process \(\{Y_t^S\}_{0 \le t \le T}\), and \(P_t^S(\cdot)\) its marginal distribution at time \(t\).
By \eqref{eq:limit}, simulating the dynamics \eqref{eq:backguide} produces a terminal sample
\[
Y_T^S \sim P_T^S(\cdot) := p^S_{\tiny\rm{pre}}(\cdot).
\]
Given a pretrained model, the only unknown component in \eqref{eq:backguide} is the guidance term
\(\nabla \log h(\cdot,\cdot)\), which enters the dynamics directly.

\subsection{Learning $h$ function}
\label{sc32}
As explained in Section~\ref{sc31}, the key challenge in sampling the guided process
\(\{Y_t^S\}_{0 \le t \le T}\) lies in learning the function \(h\), or more precisely,
its logarithmic gradient \(\nabla \log h\).
Since \(\nabla \log h = (\nabla h)/h\), a good approximator \(h_\theta\) to \(h\) does not guarantee that
\(\nabla \log h_\theta\) is close to \(\nabla \log h(\cdot,\cdot)\).
Hence, a natural strategy is to learn the numerator \(\nabla h\) and the denominator \(h\) separately.
Accordingly, Section~\ref{sc321} focuses on learning \(h\), while Section~\ref{sc322} addresses the estimation of \(\nabla h\).
The main analytical tools we employ are drawn from stochastic analysis,
in particular martingale theory and the quadratic variation of stochastic processes.

\subsubsection{Learning $h$ via martingale loss}
\label{sc321}

   By applying It\^{o}'s formula to $h(t, Y_t)$, we get:
\begin{equation}
\label{eq:hrep}
dh(t,Y_t) = \overline{g}(t) \nabla h(t,Y_t) \cdot dB_t,
\end{equation}
which leads to the following classical result. 
\begin{proposition}
\label{prop:mgle}
Let $\{Y_t\}_{0 \le t \le T}$ be the pretrained model defined by \eqref{eq:back},
and $h(\cdot, \cdot)$ be defined by \eqref{eq:hfunc}.
Then the process
$\{h(t,Y_t)\}_{0 \le t \le T}$ is a local martingale.
\end{proposition}

   A natural idea is to exploit the (local) martingale property \footnote{Local martingales are a continuous-time phenomenon \cite[Theorem II.42]{Meyer72}, i.e., there are no (strict) local martingales in discrete time. It is well known that a uniformly integrable local martingale is a (true) martingale, see \cite[Chapter II]{RY99}. 
In our setting, a sufficient condition for $\{h(t,Y_t)\}_{0 \le t \le T}$ to be a (true) martingale 
is that $\overline{g}(\cdot)$, $|\nabla h(\cdot, \cdot)|$ are bounded.} 
of $\{h(t,Y_t)\}_{0 \le t \le T}$ to learn the function $h$.
Assume that $\{h(t,Y_t)\}_{0 \le t \le T}$ is a (true) martingale.
By the $L^2$ projection of conditional expectations,
the $h$ function \eqref{eq:hfunc} {\it uniquely} solves the following optimization problem:
\begin{equation*}
\min_{\ell(\cdot, \cdot)\in \mathcal{K}} \mathbb{E}_{[0,T]} \left[\int_0^T \Big(\ell(t,Y_t) - 1(Y_T \in S)\Big)^2 dt\right].
\end{equation*}
where
\begin{eqnarray*}
    \mathcal{K}:=\Bigg\{\ell:[0,T]\times \mathbb{R}^d| \ell \text{ is Borel measurable and }\mathbb{E}\Big[\int_0^T\ell(t,Y_t)^2\Big]<\infty\Bigg\}.
\end{eqnarray*}

Now we restrict our search within a class of parametrized functions $\{h_\phi(t,y)\}_\phi$ to approximate $h(t,y)$.
This leads to the {\em martingale loss} objective to learn the $h$ function:
\begin{equation}
\label{eq:SA}
\min_{\phi} \mathbb{E}_{[0,T]}\left[ \int_0^T \Big(h_\phi(t,Y_t) - 1(Y_T \in S)\Big)^2 dt\right]
\end{equation}

\begin{remark}[Efficient off-policy learning]
Note that our learning objective is purely off-policy, in the sense that it relies only on trajectories generated by the pre-trained diffusion model $\{Y_t\}_{0\leq t\leq T}$, rather than on on-policy dynamics induced by the currently learned control or guidance mechanism. The latter evolves during training and may be unstable over time, particularly when the imposed constraints correspond to rare events. This feature fundamentally distinguishes our approach from control-based methods that rely on Doob’s $h$-transform for guidance, which inherently operate on on-policy controlled dynamics \cite{DEFT,DP24,pidstrigach2025conditioning,howard2025control}. Related off-policy ideas have recently been proposed and analyzed in \cite{MGZZ2026} for a different purpose, namely to address observability challenges in reinforcement learning under model uncertainty.
\end{remark}

Denote by $h_{\phi_*}(t,y)$ the learned $h$ function by solving the stochastic optimization problem \eqref{eq:SA}.
The algorithm for {\em conditional diffusion guidance via martingale loss} (CDG-ML)
is summarized as follows. 
\begin{algorithm}
  \caption{Conditional diffusion guidance via martingale loss (CDG-ML)} \label{algo:1}
  \begin{algorithmic}
      \State {\bf Input}: pretrained model $\{Y_t\}_{0 \le t \le T}$ \eqref{eq:back}, guidance set $S$, parametrized family $\{h_\phi(t,y)\}_\phi$
      \State    Step 1.  Solve the stochastic optimization problem \eqref{eq:SA} that outputs $\phi_*$. 
      \State    Step 2. Sample
      \begin{equation*}
      \begin{aligned}
\qquad dY^S_t &= \left(\overline{f}(t, Y^S_t) + \overline{g}(t)^2 \nabla \log h_{\phi_*}(t, Y^S_t) \right)dt + \overline{g}(t) dB_t \\
                        & = \left(\overline{f}(t, Y^S_t) + \overline{g}(t)^2 \frac{\nabla h_{\phi_*}(t, Y^S_t)}{h_{\phi_*}(t, Y^S_t)} \right)dt + \overline{g}(t) dB_t,     Y_0^S \sim p_{\tiny \rm{noise}}(\cdot).
\end{aligned}
\end{equation*}    
      \State {\bf Output}: $Y^S_T$.
  \end{algorithmic}
\end{algorithm}

   The logic of Algorithm \ref{algo:1} is that
if $h_{\phi_*}$  is a good approximation to $h$,
then so is $\nabla \log h_{\phi_*}$ to $ \nabla \log h$.
This is not always true mathematically, 
but can still serve as a simple computational proxy to $\nabla \log h$.

\subsubsection{Learning $\nabla h$ via quadratic variation}
\label{sc322}
Recall that $\nabla \log h = \frac{\nabla h}{h}$. We have seen that the denominator $h$ can be learned by means of the martingale loss \eqref{eq:SA}.
Now we explain how to learn its gradient $\nabla h$ by exploiting the quadratic variation of $\{h(t, Y_t)\}_{0 \le t \le T}$.
Recall from \eqref{eq:hrep} that
\begin{equation*}
d h(t, Y_t) = \overline{g}(t) \sum_{k = 1}^d \partial_{k} h(t, Y_t) dB^k_t,
\end{equation*}
where $\{B^k_t\}_{0 \le t \le T}$ is the $k$-th coordinate of $\{B_t\}_{0 \le t \le T}$.
Also denote by $\{Y^k_t\}_{0 \le t\le T}$ the $k$-th coordinate of the pretrained model $\{Y_t\}_{0 \le t \le T}$.
For each $k = 1, \ldots, d$, the covariation of $h(t, Y_t)$ and $Y^k_t$ is:
\begin{equation}
\label{eq:QVk}
d[h, Y^k]_t = \overline{g}(t)^2 \partial_{k} h(t, Y_t) dt,
\end{equation}
see \cite[Chapter IV]{RY99} \footnote{The equation \eqref{eq:QVk} can be understood as:
\begin{equation}
\label{eq:QVapp}
\frac{(h(t+\delta, Y_{t+\delta}) - h(t, Y_t))(Y^k_{t +\delta} - Y^k_t)}{\overline{g}(t)^2 \delta} \approx \partial_k h(t, Y_t),
   \mbox{for sufficiently small } \delta > 0. \tag{$\star$}
\end{equation}
There has been a body of works \cite{BS04, CN23, Fan05, Hoff99, Hoff992, Jacod08} on the statistical estimation of quadratic variation.
These papers consider how to estimate the quadratic variation of a process from a single trajectory
in the context of (financial) time series. 
Here we have a different scenario, where the pretrained model $\{Y_t\}_{0 \le t \le T}$ under $P_{[0,T]}(\cdot)$ can be sampled repeatedly.
The relation \eqref{eq:QVapp} naturally provides an approximation to $\partial_k h$ by regressing
the left-hand term over $(t, Y_t)$. \label{foot:QV}}.
Put it compactly,
\begin{equation}
\label{eq:QV}
d[h, Y]_t = \overline{g}(t)^2 \nabla h(t, Y_t) dt.
\end{equation}
By substituting $h$ on the right side of \eqref{eq:QV} with $h_{\phi_*}$,
and approximating $\nabla h(t,y)$ by a class of parametrized functions $\{q_\psi(t,y)\}_\psi$,
we derive the {\em covariation loss} objective:
\begin{equation}
\label{eq:SA2}
\min_\psi  \mathbb{E}_{[0,T]}\left[\int_0^T\left|\frac{1}{\overline{g}(t)^2}\frac{d[h_{\phi_*}, Y]_t}{dt} - q_\psi(t, Y_t)\right|^2 dt\right].
\end{equation}
Denote by $q_{\psi_*}(t,y)$ the learned gradient of $h$ by solving the stochastic optimization problem \eqref{eq:SA2}.
The algorithm for conditional diffusion guidance via martingale-covariation loss (CDG-MCL) is
summarized as follows.

\begin{algorithm}
  \caption{Conditional diffusion guidance via martingale-covariation loss (CDG-MCL)} \label{algo:2}
  \begin{algorithmic}
      \State {\bf Input}: pretrained model $\{Y_t\}_{0 \le t \le T}$ \eqref{eq:back}, guidance set $S$, parametrized families $\{h_\phi(t,y)\}_\phi$, $\{q_\psi(t,y)\}_\psi$
      \State    Step 1. Solve the stochastic optimization problem \eqref{eq:SA} that outputs $\phi_*$. 
       \State    Step 2. Solve the stochastic optimization problem \eqref{eq:SA2} that outputs $\psi_*$. 
      \State    Step 3. Sample
      \begin{equation*}
\qquad dY^S_t = \left(\overline{f}(t, Y^S_t) + \overline{g}(t)^2 \frac{q_{\psi_*}(t, Y^S_t)}{h_{\phi_*}(t, Y^S_t)} \right)dt + \overline{g}(t) dB_t,     Y_0^S \sim p_{\tiny \rm{noise}}(\cdot).
\end{equation*}    
      \State {\bf Output}: $Y^S_T$.
  \end{algorithmic}
\end{algorithm}

   In Algorithm \ref{algo:2}, we write $\nabla \log h = \frac{\nabla h}{h}$, and estimate the numerator $\nabla h$ and the denominator $h$ separately. 
We first learn the $h$ function via the martingale property of $\{h(t, Y_t)\}_{0 \le t \le T}$,
and then learn its gradient by using the (quadratic) covariation of $\{h(t, Y_t)\}_{0 \le t \le T}$ and $\{Y_t\}_{0 \le t \le T}$.
Denote by $\widetilde{p}^S_{\tiny \rm{pre}}(\cdot)$ the distribution of the output $Y^S_T$ in Algorithm \ref{algo:1} or \ref{algo:2}.




\section{Theoretical results}
\label{sc4}

   In this section, we provide theoretical results for the Conditional Diffusion Guidance framework introduced in Section \ref{sc2}.
The total variation and Wasserstein distance between the conditional target distribution $p^S_{\tiny \rm{data}}(\cdot)$,
and the diffusion guidance $\widetilde{p}^S_{\tiny \rm{pre}}(\cdot)$ output by Algorithm \ref{algo:1} or \ref{algo:2}
are studied in Section \ref{sc41} and \ref{sc42} respectively. 
In Section \ref{sc43}, we explore the convergence of the stochastic optimization algorithms to learn the $h$ function.

\subsection{Total variation bounds}
\label{sc41}

This part studies the total variation distance between $p^S_{\tiny \rm{data}}(\cdot)$ and $\widetilde{p}^S_{\tiny \rm{pre}}(\cdot)$.

   Recall from Section \ref{sc2} that $p(t,x)$ is the probability density of the forward process \eqref{eq:forward},
and $s_{\theta_*}(t,x)$ is the score matching function of the pretrained model. 
Also recall from Section \ref{sc31} that $P^S_t(\cdot)$ is the (marginal) distribution of $Y^S_t$ defined by \eqref{eq:backguide}.
Below we present a few assumptions.

\begin{assump}
\label{assump:1}
~
\begin{enumerate}[itemsep = 3 pt]
\item[(i)]
$d_{TV}(p(T, \cdot), \, p_{\tiny \rm noise}(\cdot)) < \infty$.
\item[(ii)]
The score matching satisfies:
$\sup_{0 \le t \le T}\mathbb{E}_{p(t, \cdot)} |s_{\theta_*}(t, X) - \nabla \log p(t, X)|^2 \le \varepsilon^2$.
\item[(iii)]
There is $\rho > 0$ such that $p_{\tiny \rm{data}}(S) \ge \rho$.
\item[(iv)]
There is $\eta > 0$ such that
$|\nabla \log h - \nabla \log h_{\phi_*}|_\infty \le \eta$ for Algorithm \ref{algo:1}, 
or $\left|\nabla  \log h - \frac{q_{\psi_*}}{h_{\phi_*}} \right|_\infty\leq \eta$ for Algorithm \ref{algo:2}.
\end{enumerate}
\end{assump}

   The assumptions $(i)$-$(ii)$ ensure that $p_{\tiny \rm{data}}(\cdot)$
and $p_{\tiny \rm{pre}}(\cdot)$ are close.
The assumption $(iii)$ indicates that the guidance set $S$ is non-negligible, so the conditional distributions on $S$ are well-defined.
The assumption $(iv)$ provides blackbox errors for learning $\nabla \log h$ \footnote{
The assumption $(iv)$ means that the function $\nabla \log h$ can be learned pointwise. 
Note that in Algorithm \ref{algo:1} and \ref{algo:2}, the $h$ function is learned by solving stochastic optimization problems
using the pretrained samples under $P_{[0,T]}(\cdot)$.
So a more ``reasonable" hypothesis is that $\nabla \log h$ can be learned under the pretrained distribution:
\begin{equation}
\label{eq:l2}
\sup_{0 \le t \le T}\mathbb{E}_t|\nabla \log h(t,Y) - \nabla \log h_{\phi_*}(t,Y)| \le \eta^2 \,\mbox{ or } \,
\sup_{0 \le t \le T} \mathbb{E}_t\left| \nabla  \log h(t,Y) - \frac{q_{\psi_*}(t,Y)}{h_{\phi_*}(t,Y)}\right|^2 \le \eta^2,
\tag{$\star\star$}
\end{equation}
which will be studied in Section \ref{sc43}.
As will be clear in the proof of Lemma \ref{lem:prepre}, 
we need (technically):
\begin{equation*}
\sup_{0\le t \le T}\mathbb{E}^S_{t}\left| \nabla \log h(t, Y) - \nabla \log h_{\phi_*}(t, Y) \right|^2 \le \eta^2 \mbox{ or }
\sup_{0 \le t \le T} \mathbb{E}^S_{t}\left| \nabla \log h(t, Y) - \frac{q_{\psi_*}(t,Y)}{h_{\phi_*}(t,Y)}\right|^2 \le \eta^2.
\end{equation*}
to establish the total variation bound.
That is, $\nabla \log h$ can be learned under the conditional guided distribution $P^S_{[0,T]}(\cdot)$.
 Our conjecture is that using sufficiently rich function approximations, 
the function $\nabla \log h$ can be learned pointwise,
so it does not matter whether the evaluation is under $P_{[0,T]}(\cdot)$ or $P^S_{[0,T]}(\cdot)$. \label{foot:S}},
which will be further developed in Section \ref{sc43}.

   First, we bound the total variation distance between $p^S_{\tiny \rm{data}}(\cdot)$
and $p^S_{\tiny \rm{pre}}(\cdot)$.

\begin{lemma}
\label{lem:predata}
Let Assumption \ref{assump:1} (iii) hold and assume $p_{\tiny \rm{pre}}(S)>0$.
We have:
\begin{equation}
\label{eq:condTV0}
d_{TV}(p^S_{\tiny \rm{pre}}(\cdot), \,p^S_{\tiny \rm{data}}(\cdot)) 
\le \frac{3}{2 \rho} d_{TV}(p_{\tiny \rm{pre}}(\cdot), \,p_{\tiny \rm{data}}(\cdot)).
\end{equation}
\end{lemma}

\begin{proof}
Without loss of generality, assume that $p_{\tiny \rm{pre}}(\cdot)$, $p_{\tiny \rm{data}}(\cdot)$ have densities.
We have:
\begin{equation}
\label{eq:condTV}
\begin{aligned}
d_{TV}(p^S_{\tiny \rm{pre}}(\cdot), \,p^S_{\tiny \rm{data}}(\cdot))
& = \frac{1}{2} \int |p^S_{\tiny \rm{pre}}(x) - p^S_{\tiny \rm{data}}(x) | dx \\
& = \frac{1}{2} \int_S \left| \frac{p_{\tiny \rm{pre}}(x)}{p_{\tiny \rm{pre}}(S)} - \frac{p_{\tiny \rm{data}}(x)}{p_{\tiny \rm{data}}(S)} \right| dx \\
& \le \frac{1}{2 p_{\tiny \rm{data}}(S)}
\left(  |p_{\tiny \rm{pre}}(S) - p_{\tiny \rm{data}}(S)| + \int_S |p_{\tiny \rm{pre}}(x) - p_{\tiny \rm{data}}(x)| dx \right) \\
& \le \frac{3}{2 p_{\tiny \rm{data}}(S)} d_{TV}(p_{\tiny \rm{pre}}(\cdot), \,p_{\tiny \rm{data}}(\cdot)),
\end{aligned}
\end{equation}
where the third inequality follows from the triangle inequality 
$|p_{\tiny \rm{pre}}(x) p_{\tiny \rm{data}}(S) - p_{\tiny \rm{data}}(x) p_{\tiny \rm{pre}}(S)| 
\le p_{\tiny \rm{pre}}(x) |p_{\tiny \rm{pre}}(S) - p_{\tiny \rm{data}}(S)| + p_{\tiny \rm{pre}}(S) |p_{\tiny \rm{pre}}(x) - p_{\tiny \rm{data}}(x)|$.
Combining \eqref{eq:condTV} with the fact that $p_{\tiny \rm{data}}(S) \ge \rho$ yields the bound \eqref{eq:condTV0}.
\end{proof}

   The next lemma bounds  the total variation distance between the conditional pretrained distribution $p^S_{\tiny \rm{pre}}(\cdot)$, and $\widetilde{p}^S_{\tiny \rm{pre}}(\cdot)$ output by Algorithm \ref{algo:1} or \ref{algo:2}.

\begin{lemma}
\label{lem:prepre}
Let Assumption \ref{assump:1} (iv) hold.
We have:
\begin{equation}
\label{eq:condprepre}
d_{TV}(p^S_{\tiny \rm{pre}}(\cdot), \,\widetilde{p}^S_{\tiny \rm{pre}}(\cdot)) \le \frac{\eta g_{\max} \sqrt{T}}{2},
\end{equation}
where $g_{\max}:=\sup_{0\le t\le T}g(t)$.
\end{lemma}
\begin{proof}
Recall that $\mu_{\phi_*}(t,y)$ denotes the function approximation for $\nabla \log h(t,y)$ in Algorithm \ref{algo:1} or \ref{algo:2}
(i.e., $\nabla \log h_{\phi_*}(t,y)$ in Algorithm \ref{algo:1} and $\frac{q_{\psi_*}(t,y)}{h_{\phi_*}(t,y)}$ in Algorithm \ref{algo:2}).
Let $\widetilde P^S_{[0,T]}$ denote the path law of $\{\widetilde Y_t^S\}_{0\le t\le T}$.
Note that 
\begin{equation}
\begin{aligned}
d_{KL}(p^S_{\tiny \rm{pre}}(\cdot), \,\tilde{p}^S_{\tiny \rm{pre}}(\cdot)) 
& \le  d_{KL}(P^S_{[0,T]}, \,\widetilde P^S_{[0,T]}) \\
& = \frac{1}{2}\mathbb{E} \int_0^T g(t)^2
\left|\nabla \log h(t, Y^S_t) - \mu_{\phi_*}(t, Y^S_t) \right|^2 dt \\
& \le \frac{1}{2}g_{\max}^2\eta^2 T,
\end{aligned}
\end{equation}
where the first inequality follows from the data processing inequality,
and the second identity is a consequence of Girsanov's theorem.
Further by Pinsker's inequality, we get
\begin{equation*}
    d_{TV}(p^S_{\tiny \rm{pre}}(\cdot), \,\tilde{p}^S_{\tiny \rm{pre}}(\cdot))
    \le
    \sqrt{\frac{1}{2}
    d_{KL}(p^S_{\tiny \rm{pre}}(\cdot), \,\tilde{p}^S_{\tiny \rm{pre}}(\cdot))}
    \le \frac{g_{\max}\eta\sqrt{T}}{2}.
\end{equation*}
This gives the bound \eqref{eq:condprepre}.
\end{proof}
\begin{proof}
Recall that $\mu_{\phi_*}(t,y)$ denotes the function approximation for $\nabla \log h(t,y)$ in Algorithm \ref{algo:1} or \ref{algo:2}
(i.e., $\nabla \log h_{\phi_*}(t,y)$ in Algorithm \ref{algo:1} and $\frac{q_{\psi_*}(t,y)}{h_{\phi_*}(t,y)}$ in Algorithm \ref{algo:2}).
Note that 
\begin{equation}
\begin{aligned}
d_{KL}(p^S_{\tiny \rm{pre}}(\cdot), \,\tilde{p}^S_{\tiny \rm{pre}}(\cdot)) 
& \le  d_{KL}(Y^S_T, \, \widetilde{Y}^S_T)  \\
& = \mathbb{E} \int_0^T \left|\nabla \log h(t, Y^S_t) - \mu_{\phi_*}(t, Y^S_t) \right|^2 dt
 \le \eta^2 T,
\end{aligned}
\end{equation}
where the first inequality follows the data processing inequality,
and the second identity is a consequence of Girsanov's theorem.
Further by Pinsker's inequality, we get the bound \eqref{eq:condprepre}.
\end{proof}

   Combining the above two lemmas yields the following result on the total variation distance between 
${p}^S_{\tiny \rm{data}}(\cdot)$ and $\widetilde{p}^S_{\tiny \rm{pre}}(\cdot)$.

\begin{theorem}
\label{thm:TV}
Let Assumption \ref{assump:1} hold.
We have:
\begin{equation}
\label{eq:maintv}
d_{TV}(p^S_{\tiny \rm data}(\cdot), \,\widetilde{p}^S_{\tiny \rm{pre}}(\cdot))
 \le \frac{3}{2 \rho} d_{TV}(p(T, \cdot), \, p_{\tiny \rm{noise}}(\cdot))  +  \left( \frac{3 \varepsilon}{2 \rho} + \eta \right) \sqrt{\frac{T}{2}}. 
\end{equation}
In particular, assuming that $\mathbb{E}_{p_{\tiny \rm{data}}(\cdot)}|X|^2 < \infty$,
there are $C_{\tiny \texttt{VE}}, C_{\tiny \texttt{VP}} > 0$ (independent of $T$) such that
\begin{equation}
d_{TV}(p^S_{\tiny \rm{data}}(\cdot), \,\widetilde{p}^S_{\tiny \rm{pre}}(\cdot)) \le 
\left\{ \begin{array}{lcl}
\frac{C_{\tiny \texttt{VE}}}{\rho} T^{-1} \sqrt{\mathbb{E}_{p_{\tiny \rm{data}}(\cdot)}|X|^2} + \left( \frac{3 \varepsilon_{\tiny \texttt{VE}}}{2 \rho} + \eta \right) \sqrt{\frac{T}{2}} & \mbox{for VE}, \\
\frac{C_{\tiny \texttt{VP}}}{\rho} e^{-C_{\tiny \texttt{VP}}\,T} \sqrt{\mathbb{E}_{p_{\tiny \rm{data}}(\cdot)}|X|^2} + \left( \frac{3 \varepsilon_{\tiny \texttt{VP}}}{2 \rho} + \eta \right) \sqrt{\frac{T}{2}}  & \mbox{for VP}.
\end{array}\right.
\end{equation}
\end{theorem}

\begin{proof}
It follows from \cite[Theorem 5.2]{TZ24tut} that
$$d_{TV}(p_{\tiny \rm{pre}}(\cdot), \,p_{\tiny \rm{data}}(\cdot)) \le 
d_{TV}(p(T, \cdot), \, p_{\tiny \rm{noise}}(\cdot)) + \varepsilon \sqrt{\frac{T}{2}}.$$
Combining this with Lemmas \ref{lem:predata} and \ref{lem:prepre}
yields the bound \eqref{eq:maintv}.
The rest of the theorem follows from \eqref{eq:maintv} and Proposition \ref{prop:VEVP}.
\end{proof}

\subsection{Wasserstein bounds}
\label{sc42}

Here we consider the Wasserstein-$2$ distance between $p^S_{\tiny \rm{data}}(\cdot)$ and $\widetilde{p}^S_{\tiny \rm{pre}}(\cdot)$,
which is more involved than the total variation bounds.
Note that we cannot bound the Wasserstein distance between $p^S_{\tiny \rm{data}}(\cdot)$ and $p^S_{\tiny \rm{pre}}(\cdot)$
in terms of that between $p_{\tiny \rm{pre}}(\cdot)$ and $p_{\tiny \rm{data}}(\cdot)$ as in Lemma \ref{lem:predata}.
Our analysis relies on coupling and Malliavin calculus. 

\medskip

Let us introduce a few more notations. Denote by $P^\circ_{[0,T]}(\cdot)$ the distribution of the (true) time reversal \eqref{eq:truetr} of $\{X_t\}_{0 \le t \le T}$ and let $\mathring{h}(t,y):= \mathbb{P}^\circ_{[0,T]}(Y_T \in S \,|\, Y_t = y)$ be the associated $h$ function.
Let $P^{\circ,S}_{[0,T]}(\cdot)$ be the conditional distribution of $P^\circ_{[0,T]}(\cdot)$ on $\{Y_T \in S\}$,
and $P^{\circ,S}_t(\cdot)$ be its marginal distribution at $Y_t\in S$.
We need the following assumptions.

\begin{assump}  We make the following assumptions and take $t\in[0,T)$ below.
\label{assump:2}
~
\begin{enumerate}[itemsep = 3 pt]
\item[(i)]
$W_2(p(T, \cdot), \, p_{\tiny \rm{noise}}(\cdot)) < \infty$.
\item[(ii)]
There is $\alpha > 0$ such that $(x-y) \cdot (f(t,x) - f(t,y)) \ge \alpha|x-y|^2$ for all $t,x,y$.
\item[(iii)]
There is $\kappa_1 > 0$ such that $p(t,\cdot)$ is $\kappa_1$-strongly log-concave for all $t$.
\item[(iv)]
There is $\varepsilon > 0$ sufficiently small: $|s_{\theta_*} - \nabla \log p|_\infty \le \varepsilon$.
\item[(v)] 
There is $\kappa_2 > 0$ such that $h(t,\cdot)$ is $\kappa_2$-strongly log-concave for all $t$.
\item[(vi)] 
There is $G: \mathbb{R}^d \to \mathbb{R}_+$ such that $|\nabla \log h(t,y)| \le \frac{G(y)}{T-t}$ for all $t,y$.
\item[(vii)]
There is $\eta > 0$ such that
$|\nabla \log h - \nabla \log h_{\phi_*}|_\infty \le \eta$ for Algorithm \ref{algo:1}, 
or $\left|\nabla  \log h - \frac{q_{\psi_*}}{h_{\phi_*}} \right|_\infty< \eta$ for Algorithm \ref{algo:2}.
\item[(viii)]
There is $F > 0$ such that $\mathbb{E}^\circ_{[0,T]}( \int_t^T |e^{\int_t^u \nabla \overline{f}(r,Y_r) dr}|_F^2 du \,|\, Y_t=y) \le F^2$ for all $t,y$.
\item[(ix)]
There is $\gamma > 0$ such that
$\mathbb{E}^\circ_{[0,T]}\left[\int_t^T |e^{\int_t^u \nabla \overline{f}(r, Y_r) dr} - e^{\int_t^u \nabla \overline{f}^\circ(r, Y_r) dr}|_F^2 du \,\Big|\, Y_t=y \right] \le \gamma^2{\color{black}(T-t)}$ for all $t,y$.
\item[(x)]
There is $K > 0$ such that $\mathbb{E}^{\circ,S}_t \left[\mathring{h}(t,Y)^{-\frac{3}{2}}\right], \, \mathbb{E}^{\circ,S}_t \left[G(Y)^2 \mathring{h}(t,Y)^{-\frac{3}{2}}\right] \le K$ for all $t$,
where $G(\cdot)$ is defined in $(vi)$.
\end{enumerate}
\end{assump}

   Before stating our result, we make several comments on Assumption \ref{assump:2}.

The conditions $(i)$--$(iv)$ can be used to bound $W_2(p_{\tiny \rm{data}}(\cdot), p_{\tiny \rm{pre}}(\cdot))$,
which together with $(vii)$ yields an estimate of $W_2(p^S_{\tiny \rm{data}}(\cdot), p^S_{\tiny \rm{pre}}(\cdot))$
involving a perturbation bound on $\nabla \log h$.
The conditions $(v)$--$(x)$ are required for the perturbation analysis of $\nabla \log h$ via Malliavin calculus.

   Note that the condition $(iii)$ holds for the VE and VP models, 
if $p_{\tiny \rm{data}}(\cdot)$ is strongly log-concave.
The condition $(iv)$ is stronger than Assumption \ref{assump:1} $(ii)$ for the same reason as explained in the footnote \footref{foot:S}.
(In fact, it suffices to assume an $L^2$ bound under the guided distribution in Algorithm \ref{algo:1} or \ref{algo:2}.)
Finally, the condition $(vi)$ is reasonable, since it holds for heat(-like) kernels.
Finally, conditions (viii) and (ix) are satisfied when $\overline{f}$ and $\overline{f}^\circ$ are differentiable in $x$ and uniformly Lipschitz.

   The following theorem provides a bound on $W_2(p^S_{\tiny \rm{data}}(\cdot), \, \widetilde{p}^S_{\tiny \rm{pre}}(\cdot))$.

\begin{theorem}\label{thm:W2}
Let Assumption \ref{assump:2} hold, and set $\Lambda:= \alpha + (\kappa_1 + \kappa_2) g_{\min}^2$ with $g_{\min}:=\inf_{0\leq t\leq T}g(t)$.
We have:
\begin{equation}
\label{eq:W2key}
W_2(p^S_{\tiny \rm{data}}(\cdot), \, \widetilde{p}^S_{\tiny \rm{pre}}(\cdot))
\le e^{-\Lambda T} W_2(p(T, \cdot), \, p_{\tiny \rm{noise}}(\cdot)) + C(\varepsilon + \eta + \gamma),
\end{equation}
for some $C > 0$ (independent of $T$).
In particular, assuming that $\mathbb{E}_{p_{\tiny \rm{data}}(\cdot)}|X|^2 < \infty$, $p_{\tiny \rm{data}}(\cdot)$ is $\kappa$-strongly log-concave 
for $\kappa$ sufficiently large and Assumption \ref{assump:2} $(iv)$--$(x)$ holds,
there are $C_{\tiny \texttt{VE}}, C_{\tiny \texttt{VP}} > 0$ (independent of $T$) such that
\begin{equation}
W_2(p^S_{\tiny \rm{data}}(\cdot), \,\widetilde{p}^S_{\tiny \rm{pre}}(\cdot)) \le 
\left\{ \begin{array}{lcl}
C_{\tiny \texttt{VE}}\left(e^{-C_{\tiny \texttt{VE}}T} \sqrt{\mathbb{E}_{p_{\tiny \rm{data}}(\cdot)}|X|^2} + \varepsilon + \eta + \gamma\right) & \mbox{for VE}, \\
C_{\tiny \texttt{VP}}\left(e^{-C_{\tiny \texttt{VP}}T} \sqrt{\mathbb{E}_{p_{\tiny \rm{data}}(\cdot)}|X|^2} + \varepsilon + \eta + \gamma\right)  & \mbox{for VP}.
\end{array}\right.
\end{equation}
\end{theorem}

\begin{remark}[Comparison to the TV bound in Theorem~\ref{thm:TV}.]
By comparing the result in Theorem \ref{thm:W2} and Theorem~\ref{thm:TV}, a central difference is that the total variation bound in Theorem~\ref{thm:TV} scales with $1/\rho$, where $\rho$ is a lower bound on $p_{\tiny \rm{data}}(S)$, whereas the Wasserstein bound does not. This reflects the intrinsic ill-conditioning of total variation when comparing conditional laws on rare events, where small unconditional modeling errors are amplified. In contrast, Wasserstein distance remains stable, making it more suitable for constrained conditional generation.

Technically, the Wasserstein bound relies on pathwise stability estimates for the underlying SDE and therefore requires stronger regularity assumptions on the drift (Assumption~\ref{assump:2}). By comparison, the total variation bound is obtained via change-of-measure and martingale arguments and holds under weaker structural assumptions (Assumption~\ref{assump:1}).

Thus, the two results are complementary: total variation provides stronger distributional guarantees under weaker assumptions, while Wasserstein yields weaker but more geometrically meaningful guarantees under stronger regularity.
\end{remark}

\begin{proof}
The proof is split into three steps. 

 \smallskip
\noindent {\bf Step 1}. 
We start by establishing a coupling bound on $W_2(p^S_{\tiny \rm{data}}(\cdot), \, \widetilde{p}^S_{\tiny \rm{pre}}(\cdot))$.
Recall that $\mu_{\phi_*}(t,y)$ denotes the function approximation for $\nabla \log h(t,y)$ in Algorithm \ref{algo:1} or \ref{algo:2}
(i.e., $\nabla \log h_{\phi_*}(t,y)$ in Algorithm \ref{algo:1} and $\frac{q_{\psi_*}(t,y)}{h_{\phi_*}(t,y)}$ in Algorithm \ref{algo:2}).

   Consider the coupled equations:
\begin{equation*}
\left\{ \begin{array}{lcl}
d U_t = \left(\overline{f}^\circ(t,U_t) +\overline{g}(t)^2 \nabla \log \mathring{h}(t, U_t) \right) dt+ \overline{g}(t) dB_t, \\
d V_t =  \left(\overline{f}(t,V_t)+ \overline{g}(t)^2 \mu_{\phi^*}(t, V_t) \right) dt+ \overline{g}(t) dB_t,
\end{array}\right.
\end{equation*}
where $(U_0, V_0)$ are coupled to achieve $W_2(p(T,\cdot), \,p_{\tiny \rm{noise}}(\cdot))$.
Note that $W^2_2(p^S_{\tiny \rm{data}}(\cdot), \, \widetilde{p}^S_{\tiny \rm{pre}}(\cdot)) \le \mathbb{E}|U_T - V_T|^2$, so our goal is to bound $\mathbb{E}|U_T - V_T|^2$. 
By It\^o's formula, we get:
\begin{multline*}
d|U_t - V_t|^2 = 2(U_t - V_t) \cdot \Big(-f(T-t, U_t) + \overline{g}(t)^2 \nabla \log p(T-t, U_t) +\overline{g}(t)^2 \nabla \log \mathring{h}(t, U_t) \\
+f(T-t, V_t) - \overline{g}(t)^2 s_{\theta_*}(T-t, V_t) -\overline{g}(t)^2 \mu_{\phi_*}(t, V_t) \Big) dt,
\end{multline*}
which implies that
\begin{equation}
\label{eq:Itodiff}
\begin{aligned}
\frac{1}{2} \frac{d \mathbb{E}|U_t - V_t|^2}{dt} &= - \underbrace{\mathbb{E}[(U_t - V_t) \cdot (f(T-t, U_t) - f(T-t, V_t))]}_{(a)} \\
& \qquad + \overline{g}(t)^2 \underbrace{\mathbb{E}[(U_t - V_t) \cdot (\nabla \log p(T-t, U_t) -  s_{\theta_*}(T-t, V_t))]}_{(b)} \\
& \qquad +  \overline{g}(t)^2 \underbrace{\mathbb{E}[(U_t - V_t) \cdot (\nabla \log \mathring{h}(t, U_t) -  \mu_{\phi_*}(t, V_t))]}_{(c)}.
\end{aligned}
\end{equation}
By Assumption \ref{assump:2} $(ii)$, the term $(a) \ge \alpha \mathbb{E}|U_t - V_t|^2$. 
For the term $(b)$, we get:
\begin{equation}
\label{eq:termb}
\begin{aligned}
(b) & = \mathbb{E}[(U_t - V_t) \cdot (\nabla \log p(T-t, U_t) - \nabla \log p(T-t, V_t))] \\
& \qquad \qquad  + \mathbb{E}[(U_t - V_t) \cdot (\nabla \log p(T-t, V_t) - s_{\theta_*}(T-t, V_t))] \\
& \le -\kappa_1 \mathbb{E}|U_t - V_t|^2 + \varepsilon \sqrt{\mathbb{E}|U_t - V_t|^2},
\end{aligned}
\end{equation}
which follows from Assumption \ref{assump:2} $(iii)$ and $(iv)$.
Similarly, we have:
\begin{equation}
\label{eq:termc}
\begin{aligned}
(c) & = \mathbb{E}[(U_t - V_t) \cdot (\nabla \log \mathring{h}(t, U_t) - \nabla \log h(t, U_t))] \\
& \qquad \qquad  + \mathbb{E}[(U_t - V_t) \cdot (\nabla \log h(t, U_t) - \nabla \log h(t, V_t))] \\
& \qquad \qquad  + \mathbb{E}[(U_t - V_t) \cdot (\nabla \log h(t, V_t) - \mu_{\phi_*}(t, V_t))] \\
& \le -\kappa_2 \mathbb{E}|U_t - V_t|^2 + \left(\eta + \sqrt{\mathbb{E}|\nabla \log \mathring{h}(t, U_t) - \nabla \log h(t, U_t))|^2} \right) \sqrt{\mathbb{E}|U_t - V_t|^2}
\end{aligned}
\end{equation}
which follows from Assumption \ref{assump:2} $(v)$ and $(vii)$.
Combining \eqref{eq:Itodiff}, \eqref{eq:termb} and \eqref{eq:termc} yields:
\begin{equation}
\label{eq:diffineq}
\begin{aligned}
\frac{d \mathbb{E}|U_t - V_t|^2}{dt} & \le
-2 \Lambda \,\mathbb{E}|U_t - V_t|^2 \\
&    + 2g_{\max}^2 \left(\sqrt{\mathbb{E}|\nabla \log \mathring{h}(t, U_t) - \nabla \log h(t, U_t))|^2} + \varepsilon + \eta \right)\sqrt{\mathbb{E}|U_t - V_t|^2}.
\end{aligned}
\end{equation}

 \smallskip
\noindent{\bf Step 2}. Throughout this step, we fix $t\in[0,T)$.
Now we apply Malliavin calculus to bound $|\nabla \log h(t,y) - \nabla \log \mathring{h} (t,y)|$.
First we consider $| h(t,y) - \mathring{h} (t,y)|$.
It follows from  \cite[Proposition 3.1]{FL99} that for $\varepsilon > 0$ sufficiently small,
\begin{equation}
\label{eq:hest}
\begin{aligned}
& |h(t,y) - \mathring{h}(t,y)| \\
&\le C \left|\mathbb{E}^\circ_{[0,T]}\left[1(Y_T \in S) \int_t^T \overline{g}(u) (\nabla \log p(T-u, Y_u) - s_{\theta^*}(T-u, Y_u)) dB_u \,\bigg| \,Y_t =y \right] \right| \\
& \le Cg_{\max} \sqrt{\mathring{h}(t, y)} \sqrt{\mathbb{E}^\circ_{[0,T]}\left[\int_t^T \Big|\nabla \log p(T-u, Y_u) - s_{\theta^*}(T-u, Y_u))\Big|^2 du \,\bigg|\, Y_t = y\right]} \\
& \le Cg_{\max}  \sqrt{T-t} \sqrt{\mathring{h}(t,y)} \,\varepsilon,
\end{aligned}
\end{equation}
where the second inequality is by the Cauchy–Schwarz inequality,
and the last inequality is due to Assumption \ref{assump:2} $(iv)$.

   Next we bound $|\nabla h(t,y) - \nabla \mathring{h}(t,y)|$.
Introduce the {\em first variation process} $\{Z_u\}_{t \le u \le T}$ which solves:
\begin{equation*}
    d Z_u = \nabla \overline{f}(u, Y_u) Z_u du,    Z_t = I.
\end{equation*}
So $Z_u = \exp \left(\int_t^u \nabla \overline{f}(r, Y_r) dr \right)$ (here $\nabla \overline{f}$ is a matrix.)
By \cite[Proposition 3.2]{FL99},
\begin{equation*}
    \nabla h(t,y) = \mathbb{E}_{[0,T]}\left(\frac{1(Y_T \in S)}{T-t} \int_t^T \frac{Z_u}{\overline{g}(u)} dB_u \right)
    = \mathbb{E}_{[0,T]}\left(\frac{1(Y_T \in S)}{T-t} \int_t^T \frac{e^{\int_t^u \nabla \overline{f}(r, Y_r) dr}}{\overline{g}(u)} dB_u \right).
\end{equation*}
A similar argument as before shows that
\begin{equation}
\label{eq:nablahest}
\begin{aligned}
& |\nabla h(t,y) - \nabla \mathring{h}(t,y)| \le (d) + (e), \mbox{ where} \\
& (d)= C \bigg| \mathbb{E}^\circ_{[0,T]}\bigg(\frac{1(Y_T \in S)}{T-t} \int_t^T \frac{e^{\int_t^u \nabla \overline{f}(r, Y_r) dr}}{\overline{g}(u)} dB_u  \\ 
& \qquad \qquad \qquad \qquad \qquad \int_t^T \overline{g}(u) (\nabla \log p(T-u, Y_u) - s_\theta(T-u, Y_u)) dB_u \,\bigg|\, Y_t = y \bigg) \bigg|,  \\
& (e) = \left|\mathbb{E}^\circ_{[0,T]}\left(\frac{1(Y_T \in S)}{T-t} \int_t^T \frac{e^{\int_t^u \nabla \overline{f}(r, Y_r) dr} - e^{\int_t^u \nabla \overline{f}^\circ(r, Y_r) dr}}{\overline{g}(u)} dB_u \, \bigg|\,Y_t = y\right) \right|.
\end{aligned}
\end{equation}
For the term $(d)$, we have:
\begin{align}
\label{eq:termd}
(d) &\le \frac{C}{(T-t) g_{\min}}  \mathring{h}(t,y)^{\frac{1}{4}} \left\{\mathbb{E}^\circ_{[0,T]}\left( \int_t^T |e^{\int_t^u \nabla \overline{f}(r,Y_r) dr}|_F^2 du \,\bigg|\, Y_t=y\right)\right\}^{\frac{1}{2}} \notag \\
& \qquad \qquad \qquad    \left\{\mathbb{E}^\circ_{[0,T]}\left(\int_t^T \overline{g}(u) (\nabla \log p(T-u, Y_u) - s_\theta(T-u, Y_u)) dB_u \,\bigg|\, Y_t=y\right)^4\right\}^{\frac{1}{4}} \notag \\
& \le \frac{C}{(T-t)g_{\min}} \varepsilon \mathring{h}(t,y)^{\frac{1}{4}} \left\{\mathbb{E}^\circ_{[0,T]}\left( \int_t^T |e^{\int_t^u \nabla \overline{f}(r,Y_r) dr}|_F^2 du \,\bigg|\, Y_t=y \right)\right\}^{\frac{1}{2}} \left(g_{\max} \varepsilon \sqrt{T-t}\right) \notag \\
& \le \frac{C g_{\max} F}{g_{\min}{\color{black}\sqrt{T-t}}} \varepsilon \mathring{h}(t,y)^{\frac{1}{4}},
\end{align}
where the first inequality is due to H\"{o}lder's inequality,
the second inequality follows from the moment inequality (see \cite[Chapter IV, $\S 4$]{RY99}) and Assumption \ref{assump:2} $(iv)$,
and the final inequality is by Assumption \ref{assump:2} $(viii)$.
For the term $(e)$, we have:
\begin{equation}
\label{eq:terme}
\begin{aligned}
(e) &\le \frac{1}{(T-t) g_{\min}} \sqrt{\mathring{h}(t,y)} \sqrt{\mathbb{E}^\circ_{[0,T]}\left( \int_t^T |e^{\int_t^u \nabla \overline{f}(r, Y_r) dr} - e^{\int_t^u \nabla \overline{f}^\circ(r, Y_r) dr}|_F^2 du \,\bigg|\, Y_t=y \right)} \\
& \le \frac{1}{g_{\min}} \frac{\gamma}{\sqrt{T-t}} \sqrt{\mathring{h}(t,y)},
\end{aligned}
\end{equation}
where the first inequality is by the Cauchy–Schwarz inequality,
and the second inequality follows from Assumption \ref{assump:2} $(ix)$.
Applying \eqref{eq:termd} and \eqref{eq:terme} into \eqref{eq:nablahest} and combining with the fact that $\mathring{h}(t,y)\leq 1$ yields:
\begin{equation}
\label{eq:nablahest2}
 |\nabla h(t,y) - \nabla \mathring{h}(t,y)| \le C  \frac{{\color{black}F \varepsilon  +}\gamma}{\sqrt{T-t}} \mathring{h}(t,y)^{\frac{1}{4}},
    \mbox{for some } C > 0. 
\end{equation}

Combining \eqref{eq:hest}, \eqref{eq:nablahest2} and Assumption \ref{assump:2} $(vi)$ leads to:
\begin{equation}
\label{eq:keyhest}
\begin{aligned}
|\nabla \log h(t,y) - \nabla \log \mathring{h}(t,y)|
& \le \left| \frac{\mathring{h}(t,y) - h(t,y)}{\mathring{h}(t,y)} \nabla \log h(t,y)\right| + \left| \frac{\nabla h(t,y) - \nabla \mathring{h}(t,y)}{\mathring{h}(t,y)} \right| \\
& \le C \left\{  \frac{{\color{black}F +}G(y)}{\sqrt{T-t}}  \varepsilon + \frac{\gamma}{\sqrt{T-t}} \right\} \mathring{h}(t,y)^{-\frac{3}{4}}.
\end{aligned}
\end{equation}

 \smallskip
\noindent{\bf Step 3}. Observe that $\{U_t\}_{0 \le t \le T}$ is distributed by $P^{\circ,S}_{[0,T]}(\cdot)$.
By \eqref{eq:diffineq}, \eqref{eq:keyhest} and Assumption \ref{assump:2} $(x)$, we have:
\begin{equation}
\frac{d \mathbb{E}|U_t - V_t|^2}{dt} \le
-{\color{black}2} \Lambda \, \mathbb{E}|U_t - V_t|^2 + C \left(\varepsilon + \eta + \frac{\varepsilon + \gamma}{\sqrt{T-t}} \right) \sqrt{\mathbb{E}|U_t -V_t|^2}.
\end{equation}
By Gr\"{o}nwall's inequality (see \cite[Theorem 21]{DR03}), we get:
\begin{equation*}
\begin{aligned}
\mathbb{E}|U_T - V_T|^2 &\le \left( e^{-\Lambda T} W_2(p(T, \cdot), \, p_{\tiny \rm{noise}}(\cdot)) + C \int_0^T \left(\varepsilon + \eta + \frac{\varepsilon + \gamma}{\sqrt{T-t}} \right)e^{- \Lambda(T-t)} dt\right)^2 \\
& \le  \left( e^{-\Lambda T} W_2(p(T, \cdot), \, p_{\tiny \rm{noise}}(\cdot)) + C (\varepsilon + \eta + \gamma)\right)^2,
\end{aligned}
\end{equation*}
which leads to the bound \eqref{eq:W2key}.

   The rest of the theorem follows from the fact that 
$W^2_2(p(T, \cdot), \, p_{\tiny \rm{noise}}(\cdot)) \le \mathbb{E}_{p_{\tiny \rm{data}}(\cdot)}|X|^2$ for VE,
and 
$W_2(p(T, \cdot), \, p_{\tiny \rm{noise}}(\cdot)) \le e^{-CT}\mathbb{E}_{p_{\tiny \rm{data}}(\cdot)}|X|^2$ with $C > 0$ for VP.
\end{proof}
 
\subsection{Learning the $h$ function}
\label{sc43}

In this section, we study the convergence of the stochastic optimization problems outlined in Section \ref{sc32} 
(Alogrithm \ref{algo:1} and \ref{algo:2}) to learn the function $\nabla \log h$.
Our approach is generic, and does not require any explicit structure of function approximations. 

\subsubsection{Learning $h$}
We first consider the convergence of the stochastic optimization problem \eqref{eq:SA}. 
The stochastic approximation to the martingale loss is given by
\begin{equation}
\label{eq:SAh}
\phi_{n+1} = \phi_n + \delta_n \mathcal{V}(\phi_n,  \tau^{(n)}, Y^{(n)}),
\end{equation}
where $\delta_n > 0$ is the step size,
$\tau^{(n)} \sim \mbox{Unif}\,[0,T]$,
$Y^{(n)} = \{Y^{(n)}_t\}_{0 \le t \le T}$ is a copy of the pretrained model $P_{[0,T]}(\cdot)$,
and  
\begin{equation}
\mathcal{V}(\phi, \tau, Y):= -2  \partial_\phi h_{\phi}(\tau, Y_{\tau}) (h_{\phi}(\tau,Y_{\tau}) - 1(Y_T \in S)).
\end{equation}

   Our goal is to provide a quantitative bound on $|h_{\phi_n}(t,y) - h(t,y)|$ (in some weak sense).
Our idea follows from \cite[Section 4]{TZreg24}, 
which relies on \cite{BMP90} for stochastic approximations.
Set 
\begin{equation}
V(\phi):= \mathbb{E}_{\tau \sim \tiny \mbox{Unif}\,[0,T]}\Big[\mathbb{E}_{[0,T]}[ \mathcal{V}(\phi, \tau, Y)]\Big].
\end{equation}
We need the following assumptions.

 \begin{assump} \label{assump:3}
~
\begin{enumerate}[itemsep = 3 pt]
\item[(i)]
The ODE $\phi'(u) = V(\phi(u))$ has a unique stable equilibrium $\phi_*$ \footnote{$\phi_*$ is the unique stable equilibrium means that 
$V(\phi) = 0$ has a unique root $\phi_*$, and $V'(\phi_*) < 0$.}.
\item[(ii)]
There is $C > 0$ such that $\mathbb{E}_{[0,T]}[\mathcal{V}(\phi_{n+1}, Y) \,|\, \phi_n] \le C(1 + \phi_n^2)$.
\item[(iii)]
There is $\ell > 0$ such that $(\phi - \phi_*) V(\phi) \le - \ell |\phi - \phi_*|^2$.
\item[(iv)]
There is  a function $\omega: \mathbb{R}_+ \to \mathbb{R}_+$ such that 
$\omega(r)/r^\nu$ is bounded for some $\nu \le 2$,
and
$|h_{\phi} - h_{\phi'}|_{\infty} \le \omega(|\phi - \phi'|)$ for all $\phi, \phi'$.
\end{enumerate}
\end{assump}

The assumptions $(i)$--$(iii)$ guarantees that the stochastic approximation \eqref{eq:SAh} converges,
and the assumption $(iv)$ quantifies the sensitivity of the function approximation $\{h_{\phi}(t,x)\}_\phi$ with respect to the parameter.

\begin{theorem}
\label{thm:learnh}
Let Assumption \ref{assump:3} hold,
and $\delta_n = \frac{A}{n^{\zeta} + B}$ for some $\zeta \le 1$, $A > \frac{\zeta}{2 \ell}$ and $B > 0$.
We have:
\begin{equation}
\label{eq:hbound}
\mathbb{E}_{[0,T]}|h_{\phi_n} - h|_\infty \le |h - h_{\phi_*}|_\infty + Cn^{-\frac{\zeta \nu}{2}},
\end{equation}
where $h$ is defined in \eqref{eq:hfunc}.
\end{theorem}

Note that the upper bound in \eqref{eq:hbound} consists of two terms. The first term, $\lVert h - h_{\phi_*}\rVert_\infty$, represents the approximation error induced by the functional class $\{h_\phi\}_{\phi}$. If the class is sufficiently expressive to contain $h$, this term vanishes. The second term converges to zero as $n \to \infty$. In particular, choosing $\zeta = 1$ and $\nu = 2$ yields a linear convergence rate of order $n^{-1}$ for the second term.

\begin{proof}
It follows from \cite[Theorem 22]{BMP90} that under Assumption \ref{assump:3} $(i)$--$(iii)$ and $\delta_n = \frac{A}{n^{\zeta} + B}$,
\begin{equation}
\label{eq:phiL2}
\mathbb{E}_{[0,T]}|\phi_n - \phi_*|^2 \le Cn^{-\zeta}.
\end{equation}
As a result,
\begin{equation}
\label{eq:happrox}
\begin{aligned}
\mathbb{E}_{[0,T]}|h_{\phi_n} - h|_\infty & \le |h - h_{\phi_*}|_\infty + \mathbb{E}_{[0,T]}|h_{\phi_n} - h_{\phi_*}|_\infty \\
& \le |h - h_{\phi_*}|_\infty + \mathbb{E}_{[0,T]}[\omega(|\phi_n - \phi_*|)]\\
& \le |h - h_{\phi_*}|_\infty + C (\mathbb{E}|\phi_n - \phi_*|^2)^{\frac{\nu}{2}},
\end{aligned}
\end{equation}
where the first inequality is from the triangle inequality,
the second inequality is due to Assumption \ref{assump:3} $(iv)$,
and the last inequality is by the Cauchy-Schwarz inequality.
Combining \eqref{eq:phiL2} and \eqref{eq:happrox} yields the bound \eqref{eq:hbound}.
\end{proof}

   The first term $|h - h_{\phi_*}|_\infty$ on the right side of \eqref{eq:hbound} quantifies how well the family $\{h_\phi(t,y)\}_\phi$ approximates the $h$ function,
and the second term $n^{-\frac{\zeta \nu}{2}}$ gives the convergence rate of the stochastic approximation \eqref{eq:SAh}.
In particular, if the family $\{h_\phi(t,y)\}_\phi$ is rich enough to contain the $h$ function (i.e., $|h - h_{\phi_*}|_\infty =0$),
and $\{h_\phi(t,y)\}_\phi$ is Lipschitz in $\phi$ (i.e, $\nu = 1$), 
then
$h_{\phi_n}$ converges to $h$ at a rate $n^{-\frac{1}{2}}$ by taking the step size $1/n$.

\subsubsection{Learning $\nabla h$}
Now we establish similar results for the stochastic optimization problem \eqref{eq:SA2}.
Fixing $n >0$, we use $h_{\phi_n}$ to approximate the covariation,
so the stochastic approximation to the covariation loss is:
\begin{equation}
\label{eq:SAnabla}
\psi_{m+1}  = \psi_{m} + \delta'_m \mathcal{U}_n(\psi_m, \tau^{(m)}, Y^{(m)}),
\end{equation}
where 
\begin{equation}
\mathcal{U}_n(\psi, \tau, Y):= -2  \partial_{\psi} q_{\psi}(\tau, Y_{\tau})\left(q_{\psi}(\tau, Y_\tau) - \frac{1}{\overline{g}(\tau)^2} \frac{d[h_{\phi_n}, Y]_t}{dt}|_{t= \tau} \right).
\end{equation}
Also set
\begin{equation}
U_n(\psi):= \mathbb{E}_{\tau \sim \tiny \mbox{Unif}\,[0,T]}\{\mathbb{E}_{[0,T]}[\mathcal{U}_n(\psi, \tau, Y)]\}.
\end{equation}
We need the following assumptions.

\begin{assump} \label{assump:4}
~
\begin{enumerate}[itemsep = 3 pt]
\item[(i)]
The ODE $\psi'(u) = U_n(\psi(u))$ has a unique stable equilibrium $\psi_{*}$.
\item[(ii)]
There is $C > 0$ such that $\mathbb{E}_{[0,T]}[\mathcal{U}_n(\psi_{m+1}, Y) \,|\, \psi_m] \le C(1 + \psi_m^2)$.
\item[(iii)]
There is $\ell > 0$ such that $(\psi - \psi_*) U_n(\psi) \le - \ell |\psi - \psi_*|^2$.
\item[(iv)]
There is  a function $\omega: \mathbb{R}_+ \to \mathbb{R}_+$ such that 
$\omega(r)/r^{\nu'}$ is bounded for some $\nu' \le 2$,
and
$|q_{\psi} - q_{\psi'}|_{\infty} \le \omega(|\psi - \psi'|)$ for all $\psi, \psi'$.
\end{enumerate}
\end{assump}

\begin{theorem}
\label{thm:learnabla}
Let Assumption \ref{assump:4} hold,
and $\delta'_m = \frac{A}{m^{\zeta'} + B}$ for some $\zeta' \le 1$, $A > \frac{\zeta'}{2 \ell}$ and $B > 0$.
We have:
\begin{equation}
\label{eq:nhbound}
\begin{aligned}
\mathbb{E}_{[0,T]}|q_{\psi_m} - \nabla h|_\infty & \le 
\mathbb{E}_{[0,T]}\left|\frac{1}{\overline{g}(t)^2} \frac{d[h_{\phi_n}, Y]_t}{dt} -\nabla h(t, Y_t) \right| + \mathbb{E}_{[0,T]}| \nabla h_{\phi_n}-q_{\psi_{*n}}|_{\infty}
+  Cm^{-\frac{\zeta' \nu'}{2}}.
\end{aligned}
\end{equation}
\end{theorem}

   The proof of the theorem is in the same vein as that of Theorem \ref{thm:learnh}.
The first term $\mathbb{E}_{[0,T]}\left(\left|\frac{1}{\overline{g}(t)^2} \frac{d[h_{\phi_n}, Y]_t}{dt} -\nabla h(t, Y_t) \right| \right)$
on the right side of \eqref{eq:nhbound}
quantifies how close the covariation $\frac{d[h_{\phi_n}, Y]_t}{dt}$ is to $\frac{d[h, Y]_t}{dt}$.
However, it is generally hard to provide an explicit bound on this term \footnote{It follows from \cite[Chapter VI, \S 6c]{JS03} (and also \cite{Jacod81}) that if a diffusion process $\{Z^n_t\}_{0 \le t \le T}$ converges in distribution to $\{Z_t\}_{0 \le t \le T}$, and under very technical conditions, the quadratic variation of $Z^n$ converges in distribution to that of $Z$.
So $\frac{d[h_{\phi_n}, Y]_t}{dt}$ and $\frac{d[h, Y]_t}{dt}$ are expected to be close, because $h_{\phi_n} \approx h$ by Theorem \ref{thm:learnh}. 
However, the proofs in \cite{Jacod81, JS03} rely on soft measure-theoretic arguments, and it seems to be a challenging task to provide an explicit convergence rate of quadratic variation.}, 
and we simply denote it by $\theta(n)$.
Also note that the estimation of $\frac{d[h_{\phi_n}, Y]_t}{dt}$ also incurs a sample error \footnote{As mentioned in the footnote \footref{foot:QV}, the quadratic variation can be estimated by sampling the pretrained model repeatedly.
So the central limit theorem implies that the sample error of estimating the covariation $\frac{d[h_{\phi_n}, Y]_t}{dt}$
is of order $1/M$, with $M$ the sample size.},
which we do not pursue here.
The second term $|\nabla h_{\phi_n}-q_{\psi_{*n}}|_{\infty}$ measures how well the family $\{q_{\psi}(t,y)\}_\psi$ approximates $\nabla h_{\phi_n}$,
and the third term $m^{-\frac{\zeta' \nu'}{2}}$ is the convergence rate of the stochastic approximation \eqref{eq:SAnabla}.

   Combining Theorem \ref{thm:learnh} and \ref{thm:learnabla},
we have (at least heuristically) that the learning error $\eta$ of $\nabla \log h$ is of order:
\begin{equation}
\theta(n) + n^{-\frac{\zeta \nu}{2}} + m^{-\frac{\zeta' \nu'}{2}} + \mbox{discrepancy of approximations } \{h_\phi(t,y)\}_\phi, \{q_\psi(t,y)\}_\psi.
\end{equation}
Again if the families $\{h_\phi(t,y)\}_\phi$, $\{q_\psi(t,y)\}_\psi$ are rich enough and Lipschitz in the parameter (i.e., $\nu = \nu' = 1$),
then $\eta$ is of order $\theta(n) + n^{-\frac{1}{2}} + m^{-\frac{1}{2}}$
by taking the step sizes $\delta_n = 1/n$, $\delta'_m = 1/m$.

\section{Extensions}
\label{sc4+}

This section provides two extensions of the algorithms proposed in Section \ref{sc3}.
In Section \ref{sc:ODE}, we consider a more efficient probability-flow ODE sampler.
In Section \ref{sc:reinforce},
we discuss how conditioning can be reinforced in the context of classifier guidance.

\subsection{ODE sampling}
\label{sc:ODE}
Diffusion models admit an equivalent probability-flow ODE whose deterministic trajectories preserve the target marginals, and ODE-based samplers are often more efficient than stochastic SDE samplers due to reduced variance accumulation, improved numerical stability, and the ability to leverage higher-order adaptive solvers.
In particular, our Diffusion Conditional Guidance framework naturally extends to ODE-based sampling with minimal modification.

The ODE sampler for the pretrained model \eqref{eq:back} follows:
\begin{equation}
\label{eq:ODEpre}
\frac{d Y_t}{dt} = -f(T-t, Y_t) + \frac{1}{2} g(T-t)^2 s_{\theta_*}(T-t, Y_t),    Y_0 \sim p_{\tiny \rm{noise}}(\cdot).
\end{equation}

The following proposition is key to the ODE sampling of $p^S_{\tiny \rm{data}}(\cdot)$.
\begin{proposition}
Let $\{X_t\}_{0 \le t \le T}$ be defined by the SDE
$$dX_t = f(t, X_t) dt + g(t) dW_t,    X_0 \sim p^S_{\tiny \rm{data}}(\cdot),$$
 and let $\{X'_t\}_{0 \le t \le T}$ be defined by the ODE
\begin{equation}
\frac{dX'_t}{dt} = f(t,X'_t) - \frac{1}{2} g^2(t)\left(\nabla \log p(t,X'_t) + \nabla \log \mathring{h}(t,X'_t) \right),   
X_0 \sim p^S_{\tiny \rm{data}}(\cdot).
\end{equation}
Then $X_t$ and $X'_t$ have the same distribution for each $t$.
\end{proposition}

\begin{proof}
Let $p(t,x; S): = \frac{\mathbb{P}(X_t \in dx, \, X_0 \in S)/dx}{p_{\tiny \rm{data}}(S)}$,
and note that
$\nabla \log p(t,x; S) = \nabla \log p(t,x) +\nabla \log \mathring{h}(t,x)$.
Thus, 
$\frac{dX'_t}{dt} = f(t,X'_t) - \frac{1}{2} \nabla \log p(t,X'_t;S)$.
The conclusion follows from \cite[Theorem 6.1]{TZ24tut}.
\end{proof}

   It is expected that $\nabla \log \mathring{h}(t,y) \approx \nabla \log h(t,y)$.
That is,
the $h$ functions are close under $P^\circ_{[0,T]}(\cdot)$ and $P_{[0,T]}(\cdot)$.
Denote by $\mu_{\phi_*}(t,y)$ the function approximation for $\nabla \log h(t,y)$ in Algorithm \ref{algo:1} or \ref{algo:2}
(i.e., $\nabla \log h_{\phi_*}(t,y)$ in Algorithm \ref{algo:1} and $\frac{q_{\psi_*}(t,y)}{h_{\phi_*}(t,y)}$ in Algorithm \ref{algo:2}).
The ODE sampler of $p^S_{\tiny \rm{data}}(\cdot)$ follows
\begin{equation}
\frac{d Y^S_t}{dt} = -f(t, Y_t) + \frac{1}{2} g(T-t)^2 \left(s_{\theta_*}(T-t, Y^S_t) +\mu_{\phi_*}(t, Y^S_t) \right),    Y_0 \sim p_{\tiny \rm{noise}}(\cdot).
\end{equation}

   Algorithm \ref{algo:1} and \ref{algo:2} can be easily adapted to labeling and ODE sampling,
which are summarized as follows.
{
\renewcommand{\thealgorithm}{A'}
\begin{algorithm}
  \caption{Conditional diffusion guidance via martingale loss (CDG-ML)} \label{algo:3}
  \begin{algorithmic}
      \State {\bf Input}: pretrained model $\{Y_t\}_{0 \le t \le T}$ \eqref{eq:ODEpre}, label $r(y)$, guidance set $S_r$, parametrized family $\{h_\phi(t,y)\}_\phi$
      \State    Step 1.  Solve the stochastic optimization problem \eqref{eq:SA} that outputs $\phi_*$. 
      \State    Step 2. Sample
      \begin{equation*}
      \frac{d Y^S_t}{dt} = -f(t, Y_t) + \frac{1}{2} \overline{g}(t)^2 s_{\theta_*}(T-t, Y^S_t) +\frac{1}{2} \overline{g}(t)^2 \nabla \log h_{\phi_*}(t, Y^S_t),    Y^S_0 \sim p_{\tiny \rm{noise}}(\cdot).
\end{equation*}    
      \State {\bf Output}: $Y^S_T$.
  \end{algorithmic}
\end{algorithm}
}
{
\renewcommand{\thealgorithm}{B'}
\begin{algorithm}
  \caption{Conditional diffusion guidance via martingale-covariation loss (CDG-MCL)} \label{algo:4}
  \begin{algorithmic}
      \State {\bf Input}: pretrained model $\{Y_t\}_{0 \le t \le T}$ \eqref{eq:back} and \eqref{eq:ODEpre}, label $r(y)$, guidance set $S_r$, parametrized families $\{h_\phi(t,y)\}_\phi$, $\{q_\psi(t,y)\}_\psi$
      \State    Step 1. Solve the stochastic optimization problem \eqref{eq:SA} that outputs $\phi_*$. 
       \State    Step 2. Use the ODE \eqref{eq:ODEpre} to sample $Y_t$, and then the SDE \eqref{eq:back} to estimate $\frac{d[h_{\phi_*}, Y]_t}{dt}$. 
       \State \qquad       \, Solve the stochastic optimization problem \eqref{eq:SA2} that outputs $\psi_*$. 
      \State    Step 3. Sample
      \begin{equation*}
\frac{d Y^S_t}{dt} = -f(t, Y_t) + \frac{1}{2} \overline{g}(t)^2 s_{\theta_*}(T-t, Y^S_t) +\frac{1}{2} \overline{g}(t)^2 \frac{q_{\psi_*}(t, Y^S_t) }{h_{\phi_*}(t, Y^S_t)},    Y^S_0 \sim p_{\tiny \rm{noise}}(\cdot).
\end{equation*}    
      \State {\bf Output}: $Y^S_T$.
  \end{algorithmic}
\end{algorithm}
}

Note that Algorithm \ref{algo:4} also requires the SDE sampler \eqref{eq:back},
which is used to estimate the covariation $\frac{d[h_{\phi_*}, Y]_t}{dt}$ in Step 2.
This is because 
the ODE and the SDE sampler only agree in distribution marginally,
but not at the level of the process that is needed to approximate the quadratic variation.

We also point out various numerical schemes to discretize the above continuous-time samplers,
see \cite[Section 5.3]{TZ24tut} and \cite{WCW24} for the references. 
Since we rely on the pretrained model for sampling,
we will simply follow its built-in schemes
(so we do not pursue this direction here).

\subsection{Reinforcing conditioning}
\label{sc:reinforce}
Central to the conditional sampling \eqref{eq:backguide} is the Baye's rule
$P_{[0,T]}(Y|Y_T \in S) \propto P_{[0,T]}(Y_T \in S|Y) P_{[0,T]}(Y)$,
which guides the sample $Y$ to a mode specified by the classifier $P_{[0,T]}(Y_T \in S|Y)$.
Classifier guidance \cite{Dh21} further strengthens conditioning 
with a guidance scale $\eta > 0$:
\begin{equation*}
P_{[0,T]}(Y |Y_T \in S) \propto P_{[0,T]}(Y_T \in S | Y)^{\eta} P_{[0,T]}(Y),
\end{equation*}
which echos  \cite{Norvig12} 
that empirical data sometimes infer physical laws differing from those assumed in traditional modeling.
See \cite{holmes2017assigningvaluepowerlikelihood} for detailed discussions.

The resulting SDE sampler is:
\begin{equation}
\label{eq:backguidew}
\begin{aligned}
dY^{S,\eta}_t &= \left(\overline{f}(t, Y^{S,\eta}_t) + \eta\overline{g}^2(t) \nabla \log h(t, Y^{S,\eta}_t) \right)dt + \overline{g}(t) dB_t.
\end{aligned}
\end{equation}
So the sampling step in CDG-ML (Algorithm A and A') reads as:
\begin{equation}
dY^{S, \eta}_t = \left(\overline{f}(t, Y^{S, \eta}_t) +  \eta\overline{g}(t)^2 \frac{\nabla h_{\phi_*}(t, Y^{S, \eta}_t)}{h_{\phi_*}(t, Y^{S, \eta}_t)} \right)dt + \overline{g}(t) dB_t,  Y_0^{S, \eta} \sim p_{\tiny \rm{noise}}(\cdot),
\end{equation}
or
\begin{equation}
\label{eq:guideA}
\frac{d Y^{S, \eta}_t}{dt} = -f(t, Y^{S, \eta}_t) + \frac{1}{2} \overline{g}(t)^2 s_{\theta_*}(T-t, Y^{S, \eta}_t) +\frac{\eta}{2} \overline{g}(t)^2 \nabla \log h_{\phi_*}(t, Y^{S, \eta}_t),  Y^{S, \eta}_0 \sim p_{\tiny \rm{noise}}(\cdot),
\end{equation}
and that in CDG-MCL (Algorithm B and B') reads as:
\begin{equation}
dY^{S, \eta}_t = \left(\overline{f}(t, Y^{S, \eta}_t) + \eta\overline{g}(t)^2 \frac{q_{\psi_*}(t, Y^{S, \eta}_t)}{h_{\phi_*}(t, Y^{S, \eta}_t)} \right)dt + \overline{g}(t) dB_t, Y_0^{S, \eta} \sim p_{\tiny \rm{noise}}(\cdot),
\end{equation}
or
\begin{equation}
\label{eq:guideB}
\frac{d Y^{S, \eta}_t}{dt} = -f(t, Y^{S, \eta}_t) + \frac{1}{2} \overline{g}(t)^2 s_{\theta_*}(T-t, Y^{S, \eta}_t) +\frac{\eta}{2} \overline{g}(t)^2 \frac{q_{\psi_*}(t, Y^{S, \eta}_t) }{h_{\phi_*}(t, Y^{S, \eta}_t)},    Y^{S, \eta}_0 \sim p_{\tiny \rm{noise}}(\cdot).
\end{equation}

A crucial problem in classifier guidance is the choice of the scale $\eta$.
It was shown in \cite{WC24} that as $\eta$ increases,
the generation is steered toward the high-probability modes induced by the classifier,
while the sample diversity (in entropy) decreases.
On the other hand, 
overly large  $\eta$ will lead to mode collapse.

\section{Numerical experiments}
\label{sc5}
In this section, we present numerical experiments to illustrate and validate the proposed Conditional Diffusion Guidance framework. We test our framework on a synthetic example to visualize the power of enforcing a sharp constraint. We then apply our framework to perform stress testing in finance and supply chain systems. 
Across all three testcases, ur experiments are designed to assess (i) the effectiveness of the martingale-based objectives for learning the guidance function $h$ and its gradient, (ii) the quality of the resulting conditional samples compared with the target conditional distribution, and (iii) the empirical behavior of the method under finite-sample and model-misspecification effects. In particular, we focus on demonstrating how the proposed algorithms translate the theoretical guarantees developed in Section 4 into practical performance, and how different guidance strategies impact sample fidelity and stability. All experiments are conducted using pretrained diffusion models, emphasizing that the guidance procedure operates as a lightweight post-training mechanism without modifying the underlying score network.
\subsection{Synthetic examples}
\label{sc61}
To demonstrate the effectiveness of our framework, we consider one-dimensional and two-dimensional toy examples.

First, suppose we aim to generate data from the Gaussian distribution $\mathcal{N}(1,4)$, and define the endogenous guidance set as $S = (3,\infty)$. When the target distribution is Gaussian, its marginal distributions along the diffusion process remain Gaussian (e.g., under a variance-exploding SDE of the form $d X_t = \sigma^t dW_t, t\in[0,1]$). As a result, the corresponding score functions are available in closed form, and we do not need to train a neural network to estimate them via score matching.


   We evaluate Algorithm~\ref{algo:1}, \emph{Conditional Diffusion Guidance via Martingale Loss} (CDG-ML), and Algorithm~\ref{algo:2}, \emph{Conditional Diffusion Guidance via Martingale–Covariation Loss} (CDG-MCL). The corresponding results are reported in Figures~\ref{fig:1cdg-ml} and~\ref{fig:1cdg-mcl}, respectively. To further quantify the agreement between the generated samples and the target distribution, we conduct a Kolmogorov–Smirnov (K–S) test. CDG-ML yields a K–S statistic of $0.0694$ with a $p$-value of $7.1 \times 10^{-126}$, while CDG-MCL yields a K–S statistic of $0.0437$ with a $p$-value of $4.1 \times 10^{-50}$. In both cases, the relatively small K–S statistics indicate that the empirical distributions closely match the target conditional distribution $\mathcal{N}(1,4)\mid(3,\infty)$, with CDG-MCL exhibiting a closer fit.

\begin{figure}[htbp]
    \centering
    \begin{subfigure}{0.5\linewidth}
        \centering
        \includegraphics[width=\linewidth]{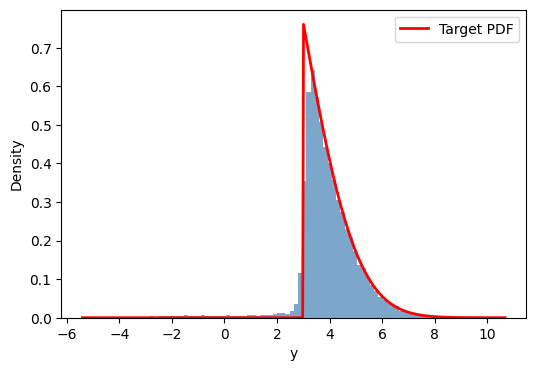}
        \caption{CDG-ML on $\mathcal{N}(1,4)|(3,\infty)$}
        \label{fig:1cdg-ml}
    \end{subfigure}%
    \hfill
    \begin{subfigure}{0.5\linewidth}
        \centering
        \includegraphics[width=\linewidth]{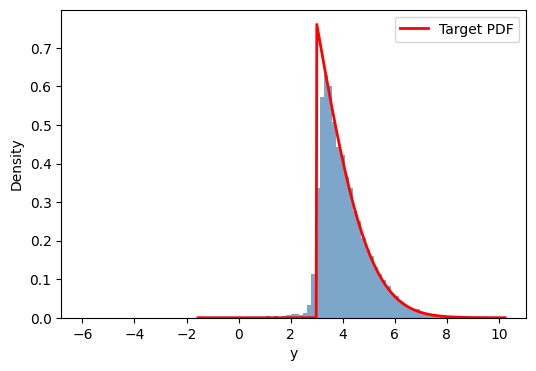}
        \caption{CDG-MCL on $\mathcal{N}(1,4)|(3,\infty)$}
        \label{fig:1cdg-mcl}
    \end{subfigure}
\end{figure}

Next, we consider a two-dimensional synthetic example with target distribution $\mathcal{N}(0,4I_2)$ and endogenous guidance set $S = (1,\infty) \times (1,\infty)$. As in the one-dimensional case, the marginal distributions along the diffusion process admit closed-form expressions. The histogram of samples from the true conditional distribution is shown in Figure~\ref{fig:2d-real}. To quantitatively assess performance, we compute the 2-Wasserstein distance. Algorithm~\ref{algo:1} yields a distance of $0.3451$, while Algorithm~\ref{algo:2} achieves a substantially smaller distance of $0.0765$. The corresponding generated histograms are presented in Figures~\ref{fig:2cdg-ml} and~\ref{fig:2cdg-mcl}, respectively. Both methods produce reasonable approximations of the target conditional distribution $\mathcal{N}(0,4I_2)\mid\big((1,\infty)\times(1,\infty)\big)$, with Algorithm~\ref{algo:2} providing a closer match.


\begin{figure}
    \centering
    \includegraphics[width=0.5\linewidth]{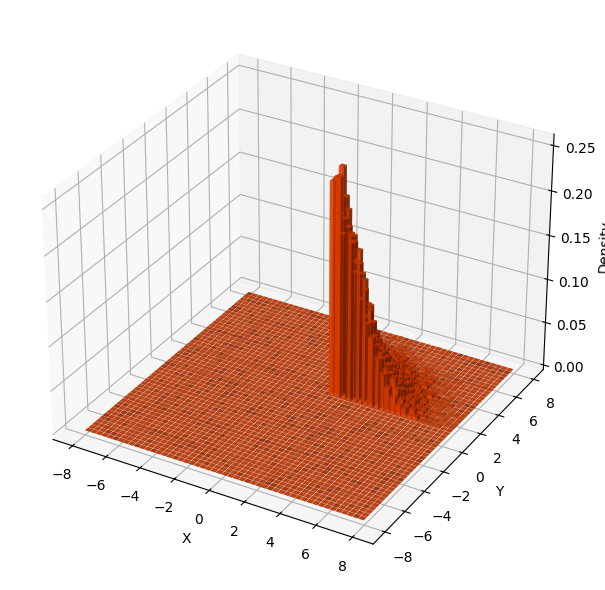}
    \caption{Histogram of Truncated Normal $\mathcal{N}(0,4I_2)|(1, \infty) \times (1, \infty)$}
    \label{fig:2d-real}
\end{figure}

\begin{figure}[htbp]
    \centering
    \begin{subfigure}{0.5\linewidth}
        \centering
        \includegraphics[width=\linewidth]{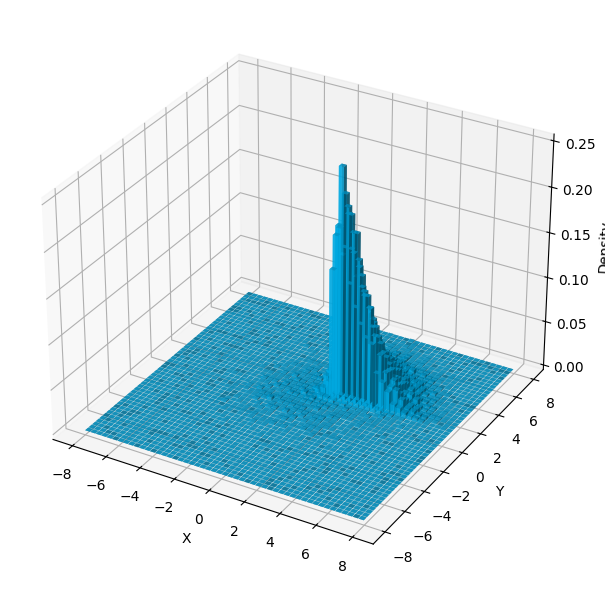}
        \caption{Generated histogram via CDG-ML}
        \label{fig:2cdg-ml}
    \end{subfigure}%
    \hfill
    \begin{subfigure}{0.5\linewidth}
        \centering
        \includegraphics[width=\linewidth]{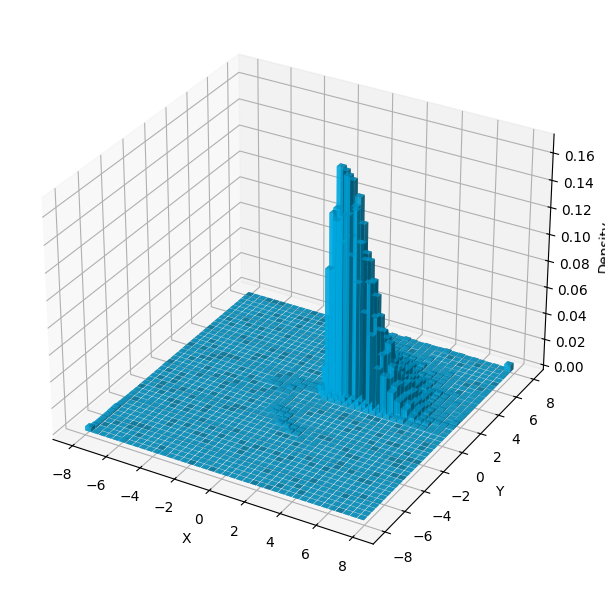}
        \caption{Generated histogram via CDG-MCL}
        \label{fig:2cdg-mcl}
    \end{subfigure}
\end{figure}

\subsection{Stress Testing}
\label{sc62}
In this section, we illustrate how our Conditional Diffusion Guidance framework can produce controlled perturbations of high-dimensional return distributions, allowing one to probe the sensitivity of portfolio-level risk and tail behavior under targeted distributional shifts while preserving realistic dependence structures.
Specifically, we apply our CDG-ML and CDG-MCL algorithms to real-world financial data and evaluate its economic relevance in constructing \cite{guo2025diffusiongenerativemodelsmeet}: (1) equal weight portfolios, (2) Markowitz minimum variance portfolios \cite{Portfolio1952},
and (3) risk parity portfolios \cite{Qian2005RiskParity}. 

We use daily log-return data for U.S. stocks: AAPL, AMZN, TSLA, JPM from August 25th, 2016, to October 19th, 2023 as the training set. We applied a mild 0.5\% winsorization to the data, removed weekday (Monday–Friday) seasonal effects, and standardized the series by dividing by its long-term standard deviation. We use standardized daily log-return data to construct a joint dataset of $d$ stocks over a rolling window of $N$ consecutive days. Each window is shifted by one day to create overlapping samples. 
In our setting, $N = 64$ (approximately three months) and $d = 4$.

\renewcommand{\thealgorithm}{E}
\begin{algorithm}[H]
  \caption{Portfolio Construction and Evaluation under Guidance Conditions}
  \label{algo:portfolio_eval}
  \begin{algorithmic}
    \State \textbf{Input:} Standardized daily log-return data $\{r_{t,i}\}$ for $d$ stocks over $T$ days;
           window size $N$ (e.g., 64); guidance window $k$ (e.g., 10);
           evaluation window $m$ (e.g., 5); threshold $\tau$ for guidance selection.
    \State \textbf{Output:} Cumulative return comparison between generated and real portfolios.
    \Statex

    \State \textbf{Step 1. Data Preparation:}
      \For{$t = 1$ to $T - N + 1$}
        \State Construct a rolling window dataset
        $\mathbf{R}_t = [r_{t,i}, r_{t+1,i}, \dots, r_{t+N-1,i}]_{i=1}^d$.
      \EndFor

    \Statex
    \State \textbf{Step 2. Guidance Set Selection:}
      \State Select a subset of stocks $\mathcal{G}$ such that their cumulative
      log-returns over the last $k$ days satisfy
      $\sum_{j=N-k+1}^{N} r_{t+j,i} < \tau, \ \forall i \in \mathcal{G}$.

    \Statex
    \State \textbf{Step 3. Portfolio Optimization:}
      \For{each model $\in \{\text{CDG-ML}, \text{CDG-MCL}\}$}
        \State Generate synthetic data $\hat{\mathbf{R}}^{(\text{model})} \in \mathbb{R}^{N \times d}$.
        \State Apply the following portfolio strategies on the first $N - k$ days:
        \begin{enumerate}
          \item Equal-weight portfolio: $w_i = 1/d$
          \item Markowitz mean-variance portfolio
          \item Risk-parity portfolio
        \end{enumerate}
      \EndFor

    \Statex
    \State \textbf{Step 4. Evaluation:}
      \State For each portfolio, compute the cumulative return $R_{\text{cum}}$ over the final $m$ days ($m \le k$):
      \State Compare $R_{\text{cum}}$ under synthetic data with that under real market conditions where
      $\sum_{j=N-k+1}^{N} r_{t+j,i} < \tau$ for $i \in \mathcal{G}$.
  \end{algorithmic}
\end{algorithm}

We choose the \textit{guidance set $S$}  to represent a stress scenario; for example, $S$ may consist of paths where the cumulative log return of selected stocks over the last $k$ days falls below a prescribed threshold. In our setting, we take $S$ to be the event that TSLA’s cumulative log return over the last 10 days is less than $-5\%$. 
We then apply three portfolio management strategies: (1) equal weight, (2) Markowitz minimum variance, and (3) risk parity on the first $N - k$ days of the $N \times d$ data generated by CDG-ML and CDG-MCL, respectively.

The resulting portfolios are evaluated by comparing their cumulative returns over the last $m$ days ($m \leq k$) with those observed under real market conditions, 
where the same subset of stocks exhibits cumulative log returns below the threshold. 
In our empirical example, we set $k = 10$ and $m = 5$. The detailed pipeline is shown in Algorithm \ref{algo:portfolio_eval}.

For the inference (sampling) process, we employ the DDIM sampler, a deterministic sampler obtained by discretizing the probability-flow ODE of the diffusion model (see \cite[Section 5.1]{TZ24tut}).
As discussed in Section \ref{sc:reinforce},
we adopt a guidance scale $\eta >0$ to adjust the magnituide of guidance,
with the sampler \eqref{eq:guideA} or \eqref{eq:guideB}.
Further, we take 
$$f(t,x) = -\frac{1}{2}\beta(t)x \quad \mbox{and} \quad 
g(t) = \sqrt{\beta(t)},$$
with linear noise schedules $\beta(t) = \beta_{\text{min}} + (\beta_{\text{max}}-\beta_{\text{min}})\,t$ for $0 \le t \le 1$. 

Moreover, using the idea that the step size should be scheduled according to the noise level \cite{song2020improvedtechniquestrainingscorebased}, we build the VP-SDE time grid by spacing time points according to the noise level. This produces a decreasing time grid that allocates smaller steps when approaching the final stage.

\paragraph{In-sample performance.} Figures~\ref{fig:finance_autograd} and~\ref{fig:finance_qmodel} show the performance of CDG-ML and CDG-MCL in recovering the cumulative return over the last $k=5$ days, conditional on the stress scenario that TSLA’s cumulative log-return over the last $10$ days is below $-5\%$. The in-sample reference samples, shown by the orange histograms, are simulated from pre-trained dynamics that already satisfy the constraint, rather than from the real data. 
This is because most proprietary pretrained models do not disclose their actual training data,
so we can only query the pretrained model to approximate the in-sample distribution.
Also this eliminates errors from the pre-trained model, which are not the focus of our post-training framework.

\begin{figure}[H]
    \centering
    \includegraphics[width=0.9\linewidth]{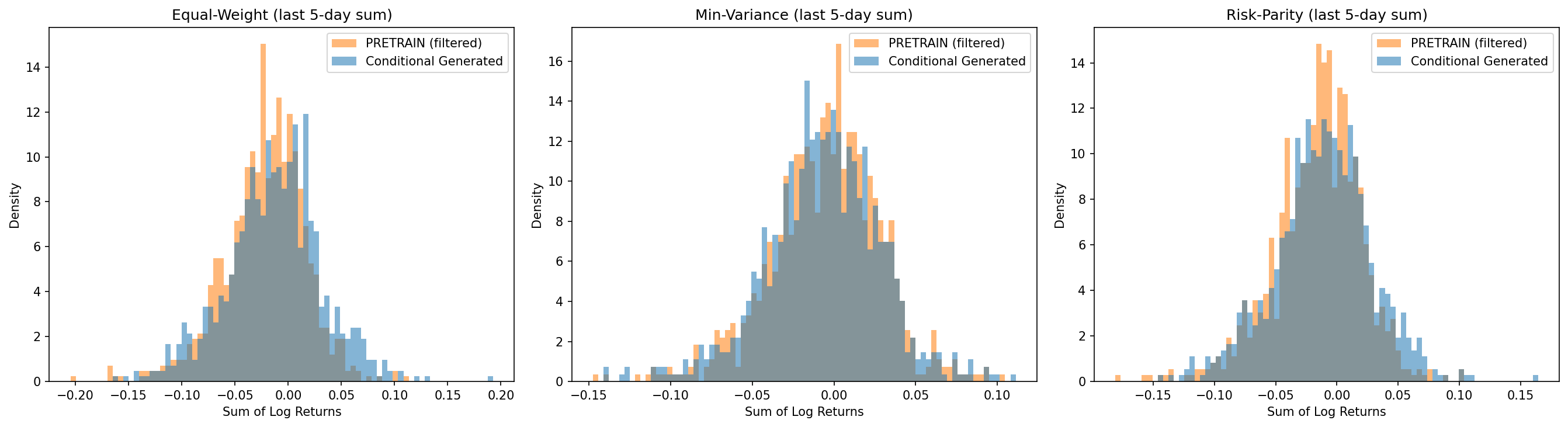}
    \caption{Histograms of portfolio constructions via CDG-ML ($\eta = 5$).}
    \label{fig:finance_autograd}
\end{figure}

\begin{figure}[H]
    \centering
    \includegraphics[width=0.9\linewidth]{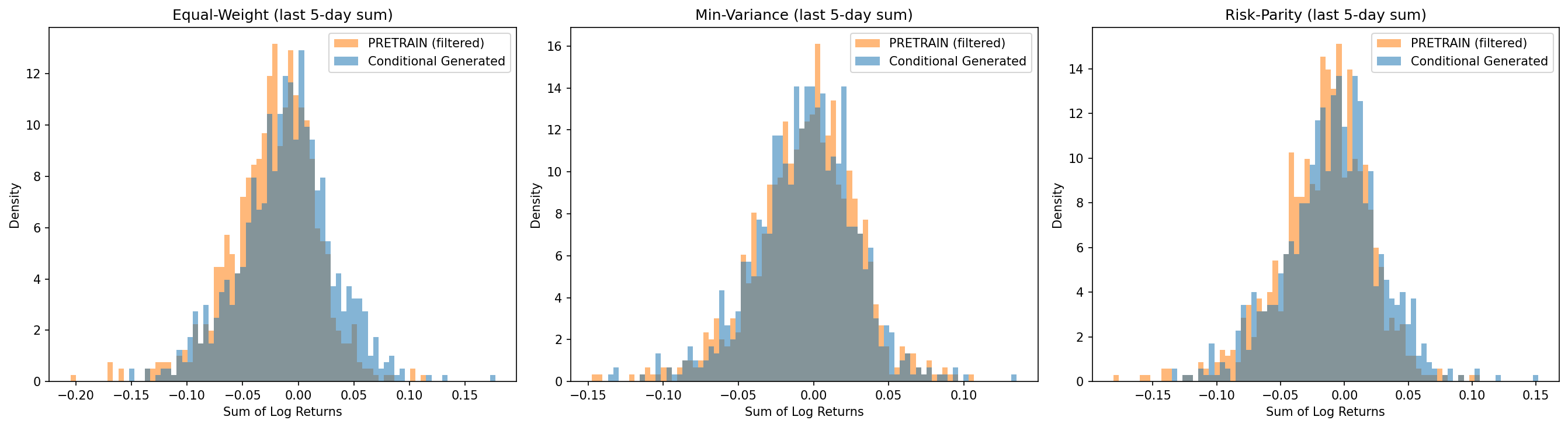}
    \caption{Histograms of portfolio constructions via CDG-MCL ($\eta = 2$).}
    \label{fig:finance_qmodel}
\end{figure}

Since the guidance strength is controlled by the tunable hyperparameter $\eta$, we select $\eta$ separately for each method. Empirically, CDG-ML uses a large guidance strength, with $\eta=5$ in our experiments, whereas CDG-MCL allows relatively smaller values, around $\eta=2$. 
This behavior is expected, because CDG-ML approximates $\nabla h$ by applying auto-differentiation to the learned surrogate $h_\phi$, and this yields a less accurate guidance to the conditional distribution than the direct gradient estimation used in CDG-MCL. Detailed results are reported in Table~\ref{tab:cdg_summary_insample_stoch05}. Within each method, CDG-ML and CDG-MCL, bold entries indicate the statistics that are closest to those of the pre-trained data across the three portfolio constructions, while underlined entries denote the second closest. 

\begin{table}[H]
\centering
\small

\begin{tabular}{lccccc}
\toprule
Portfolio & $\eta$ & Mean & Std & 5\% Quantile & 10\% Quantile \\
\midrule
\multicolumn{6}{c}{\textbf{Pretrained Generation (satisfying constraint)}} \\
\midrule
Equal Weight & -- & -0.02210 & 0.03869 & -0.08714 & -0.06975 \\
Min Variance & -- & -0.00513 & 0.03391 & -0.06483 & -0.04706 \\
Risk Parity & -- & -0.01566 & 0.03581 & -0.07775 & -0.06176 \\

\midrule
\multicolumn{6}{c}{\textbf{CDG-ML}} \\
\midrule

\multirow{5}{*}{Equal Weight}
& 0.5 & -0.00188 & \textbf{0.04045} & -0.06871 & -0.05145 \\
& 1 & -0.00011 & \underline{0.04113} & -0.07221 & -0.05404 \\
& 1.5 & -0.00208 & 0.04176 & -0.07549 & -0.05693 \\
& 2 & \underline{-0.00402} & 0.04235 & \underline{-0.07800} & \underline{-0.05931} \\
& 5 & \textcolor{black}{\textbf{-0.01339}} & 0.04483 & \textbf{-0.09380} & \textcolor{black}{\textbf{-0.07198}} \\
\noalign{\vskip 0.05em}
\hdashline
\noalign{\vskip 0.05em}

\multirow{5}{*}{Min Variance}
& 0.5 & 0.00269 & 0.03150 & -0.05125 & -0.03582 \\
& 1 & 0.00139 & 0.03193 & -0.05263 & -0.03804 \\
& 1.5 & 0.00011 & 0.03234 & -0.05509 & -0.04004 \\
& 2 & \underline{-0.00113} & \underline{0.03271} & \underline{-0.05675} & \underline{-0.04156} \\
& 5 & \textbf{-0.00710} & \textbf{0.03504} & \textcolor{black}{\textbf{-0.06698}} & \textbf{-0.04970} \\
\noalign{\vskip 0.05em}
\hdashline
\noalign{\vskip 0.05em}

\multirow{5}{*}{Risk Parity}
& 0.5 & 0.00204 & \underline{0.03547} & -0.05928 & -0.04349 \\
& 1 & 0.00029 & \textbf{0.03607} & -0.06081 & -0.04592 \\
& 1.5 & -0.00146 & 0.03663 & -0.06366 & -0.04896 \\
& 2 & \underline{-0.00317} & 0.03716 & \underline{-0.06660} & \underline{-0.05145} \\
& 5 & \textcolor{black}{\textbf{-0.01167}} & 0.03951 & \textcolor{black}{\textbf{-0.08021}} & \textcolor{black}{\textbf{-0.06280}} \\

\midrule
\multicolumn{6}{c}{\textbf{CDG-MCL}} \\
\midrule

\multirow{5}{*}{Equal Weight}
& 0.5 & 0.00015 & \textcolor{black}{\textbf{0.03996}} & -0.07126 & -0.05215 \\
& 1 & -0.00361 & \underline{0.04043} & -0.07594 & -0.05620 \\
& 1.5 & -0.00742 & 0.04105 & \underline{-0.08094} & \underline{-0.06158} \\
& 2 & \textbf{-0.01126} & 0.04174 & \textcolor{black}{\textbf{-0.08490}} & \textbf{-0.06599} \\
& 5 & \underline{-0.03411} & 0.04522 & -0.11453 & -0.09527 \\
\noalign{\vskip 0.05em}
\hdashline
\noalign{\vskip 0.05em}

\multirow{5}{*}{Min Variance}
& 0.5 & 0.00158 & 0.03133 & -0.05215 & -0.03634 \\
& 1 & -0.00085 & 0.03180 & -0.05507 & -0.03969 \\
& 1.5 & \underline{-0.00328} & \underline{0.03238} & \underline{-0.05798} & \underline{-0.04308} \\
& 2 & \textcolor{black}{\textbf{-0.00572}} & \textcolor{black}{\textbf{0.03302}} & \textbf{-0.06128} & \textcolor{black}{\textbf{-0.04648}} \\
& 5 & -0.02044 & 0.03678 & -0.08283 & -0.06395 \\
\noalign{\vskip 0.05em}
\hdashline
\noalign{\vskip 0.05em}

\multirow{5}{*}{Risk Parity}
& 0.5 & 0.00050 & 0.03512 & -0.06077 & -0.04501 \\
& 1 & -0.00283 & \textcolor{black}{\textbf{0.03563}} & -0.06432 & -0.05023 \\
& 1.5 & \underline{-0.00622} & \underline{0.03626} & \underline{-0.06870} & \underline{-0.05353} \\
& 2 & \textbf{-0.00965} & 0.03695 & \textbf{-0.07387} & \textbf{-0.05794} \\
& 5 & -0.03039 & 0.04070 & -0.09984 & -0.08485 \\

\bottomrule
\end{tabular}
\caption{Comparison of different portfolio constructions under CDG-ML and CDG-MCL, using samples generated by the pre-trained model and the conditional guidance models. \textbf{Bold} entries indicate the statistics closest to those of the pre-trained generation under each portfolio construction, while \underline{underlined} entries denote the second closest. } 
\label{tab:cdg_summary_insample_stoch05}
\end{table}

\paragraph{Out-of-sample comparison.} In addition to evaluating the in-sample coverage of our conditional guidance framework, we assess its out-of-sample performance. Specifically, we use data from August 25, 2016 to October 19, 2023 for training, and data from October 20, 2023 to March 13, 2026 for testing, covering 600 trading days. During the test period, we identify 203 windows that satisfy the stress condition defined above.

Figures~\ref{fig:oos_autograd} and~\ref{fig:oos_qmodel} present the out-of-sample comparison of CDG-ML and CDG-MCL under new market conditions, with the guidance strength $\eta$ fixed at the value selected from the in-sample training period. The corresponding portfolio-level statistics computed from real market data are reported in Table~\ref{tab:cdg_summary_outsample_stoch05}.

\begin{figure}[H]
    \centering
    \includegraphics[width=0.9\linewidth]{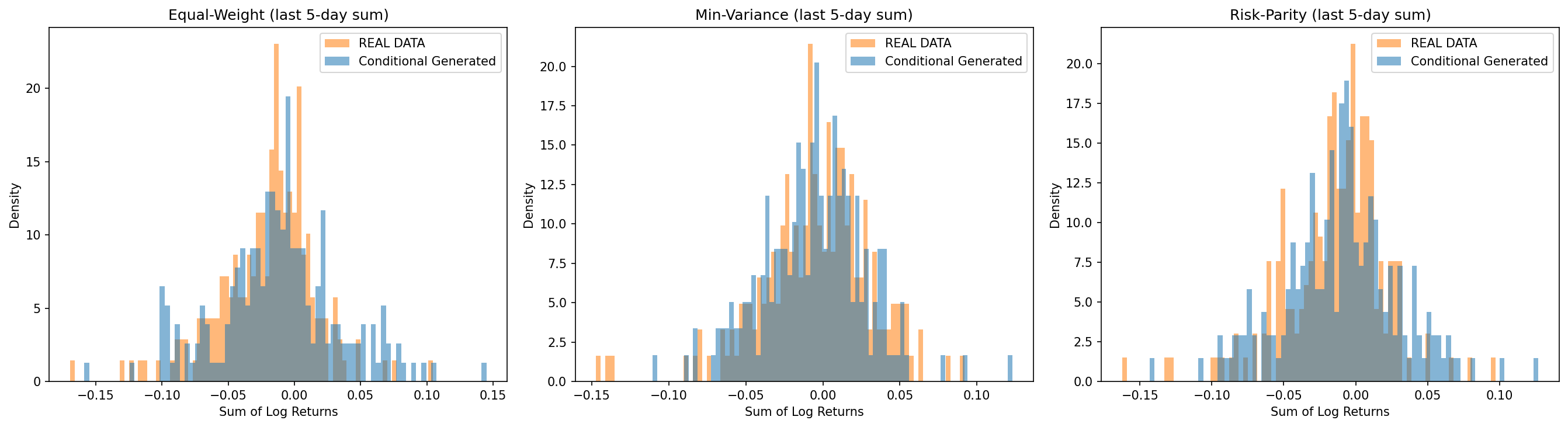}
    \caption{Out-of-sample performance of CDG-ML ($\eta = 5$). }
    \label{fig:oos_autograd}
\end{figure}

\begin{figure}[H]
    \centering
    \includegraphics[width=0.9\linewidth]{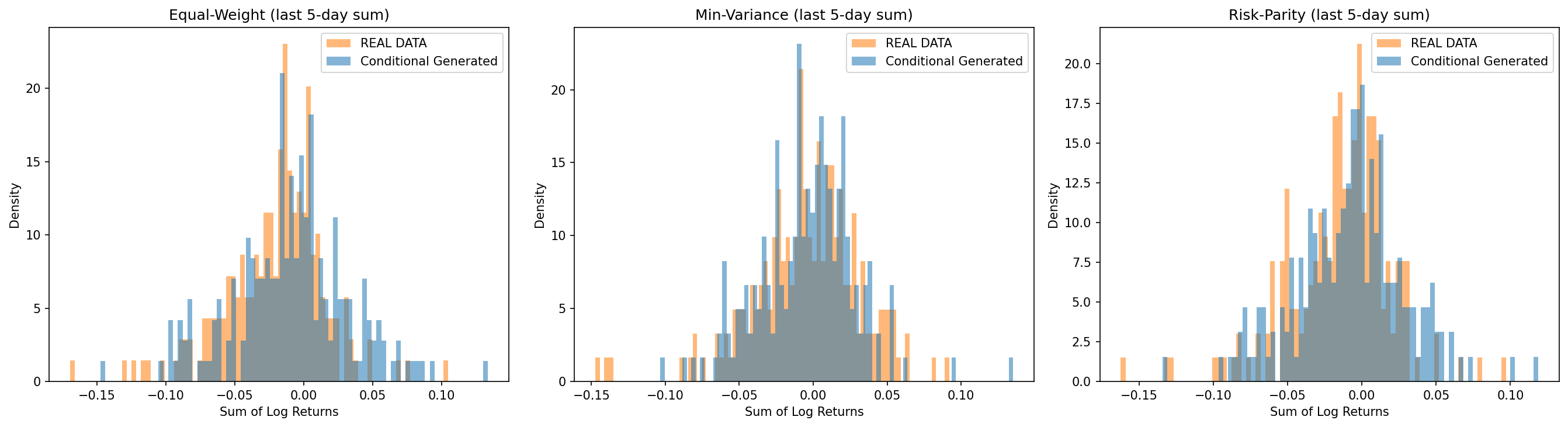}
    \caption{Out-of-sample performance of CDG-MCL ($\eta = 2$).}
    \label{fig:oos_qmodel}
\end{figure}

\begin{table}[H]
\centering
\small
\begin{tabular}{lccccc}
\toprule
Portfolio & $\eta$ & Mean & Std & 5\% Quantile & 10\% Quantile \\
\midrule
\multicolumn{6}{c}{\textbf{Real Data}} \\
\midrule
Equal Weight & -- & -0.02072 & 0.03627 & -0.08591 & -0.06645 \\
Min Variance & -- & -0.00450 & 0.03626 & -0.06278 & -0.04685 \\
Risk Parity & -- & -0.01427 & 0.03422 & -0.07139 & -0.05309 \\

\midrule
\multicolumn{6}{c}{\textbf{CDG-ML}} \\
\midrule
Equal Weight & 5 & -0.01124 & 0.04562 & -0.09392 & -0.06910 \\
Min Variance & 5 & -0.00616 & 0.03308 & -0.06111 & -0.04951 \\
Risk Parity & 5 & -0.00987 & 0.03968 & -0.07834 & -0.06217 \\

\midrule
\multicolumn{6}{c}{\textbf{CDG-MCL}} \\
\midrule
Equal Weight & 2 & -0.00879 & 0.04177 & -0.08469 & -0.06500 \\
Min Variance & 2 & -0.00465 & 0.03183 & -0.05978 & -0.04655 \\
Risk Parity & 2 & -0.00772 & 0.03675 & -0.07151 & -0.05498 \\

\bottomrule
\end{tabular}
\caption{Different portfolios under CDG-ML ($\eta=5$) and CDG-MCL ($\eta=2$) using out-of-sample real data and data generated by conditional guidance model. }
\label{tab:cdg_summary_outsample_stoch05}
\end{table}

We observe that, even under new market conditions, both algorithms achieve reasonably good coverage of the modes. In particular, for tail-risk measures such as the $5\%$ and $10\%$ quantiles, the errors are below $1\%$ across all portfolio constructions under both algorithms.

\subsection{Supply Chain Simulation}
\label{sec63}
 In this section, we demonstrate how conditional diffusion guidance can be used as a principled, data-driven tool for generating stress scenarios in complex supply-chain and queueing systems, enabling downstream simulation-based evaluation and capacity planning.

To start, we first demonstrate the performance of our framework in simulating conditional event arrivals. Note that one advantage of using our conditional diffusion guidance is that we can tune the degree of softness by varying $\eta$, and generating samples that are empirically biased to guidance set but still maintain diversity to some extent. As shown in Figure \ref{fig:con_exp}, although we generate exponential(1) data conditioning on $(1, +\infty)$, it still admits a small proportion (around 9\%) of samples in $(0,1)$, which is more realistic compared to hard conditioning.
\begin{figure}[H]
    \centering
    \includegraphics[width=0.6\linewidth]{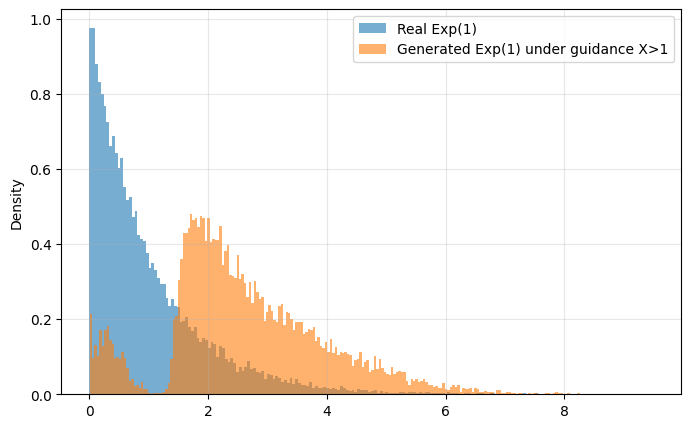}
    \caption{Density Comparison of Exp(1) and $\text{Exp}(1)|(1, \infty)$}
    \label{fig:con_exp}
\end{figure}

Next, we test our framework on generating conditional arrival times and service times in queueing networks.
We consider the hospital environment provided by \texttt{QGym} \cite{chen2024qgymscalablesimulationbenchmarking}, which consists of $M=8$ patient types (Cardiology, Surgery, Orthopedics, Respiratory disease, Gastroenterology and endoscopy, Renal disease, General Medicine, Neurology) and $N=11$ wards, each ward having a fixed number of beds corresponding to the available servers. Admission is subject to routing constraints: each ward can serve only 1–2 specific patient types, and patients who cannot be admitted immediately wait in the appropriate queue. When a bed becomes available, the ward follows a classic policy called the
$c\mu$ assignment rule \cite{cox1961queues}:
the action (i.e., job-server assignment) is decided by
\[
    \max_{a \in \mathcal{A}} \sum_{1\leq i\leq M, 1\leq j\leq N} c_i \mu_{ij} a_{ij}
\]
where $\mathcal{A} \subset \mathbb{R}^{M \times N}$ captures the routing constraints mentioned above. In this case, server $j$ prioritizes the queue with a larger $c_i \mu_{ij}$-index, where $c_i$ denotes the holding cost per job per unit time for queue $i$ $(1\leq i \leq M)$, and $\mu_{ij}$ is the service rate when server $j$ $(1\leq j \leq N)$ processes a job from queue $i$.

Based on the literature \cite{anthony2017seasonalSSI, do2023winterInHospitalCardiacArrest, kaufman2022impactHeatKidneyStones}, there is consistent evidence across temperate climate zones (e.g., North America, Europe, and North China) that healthcare burdens exhibit strong seasonal concentration, peaking in winter and, for certain specialties such as surgery and surgical site infections, also during summer or high-temperature periods. In the context of a conditional diffusion model, this periodic structure can be naturally exploited by {\it designing  guidance} that allocates higher arrival intensities and slower service rates—reflecting congestion—during these critical seasonal windows. As a representative stress scenario, we consider the flu season, in which elevated arrival rates combined with reduced service capacity imply that more patients enter the system while recovery takes longer, leading intuitively to a pronounced increase in total queue length.

Using the default hospital setting in the \texttt{QGym} environment \cite{chen2024qgymscalablesimulationbenchmarking}, the $c\mu$-rule yields a stable total queue length (aggregated over 8 patient types) that converges to approximately $400$ (Figure~\ref{fig:orig_queue}). Let $Z_1\sim\mathrm{Exp}(\lambda_1)$ and $Z_2\sim\mathrm{Exp}(\lambda_2)$ denote the inter-arrival and service times, respectively. To model seasonal effects, we condition the arrivals and services as $(Z_1\mid Z_1<\frac{3}{2 \lambda_1})$ and $(Z_2\mid Z_2 > \frac{1}{2\lambda_2})$, which induces an unstable queueing regime. 

Here, we implement the diffusion process using a drift-free VE–SDE with exponentially increasing diffusion 
$g(t)=\sigma^t$, and discretize time on a non-uniform grid constructed so that the marginal perturbation standard deviation increases linearly across steps, yielding an analytically tractable schedule that improves numerical stability during sampling. In addition, we set $\lambda_1=\lambda_2=1$
 and choose $\eta\approx 12$.

Under this setting, both the hard truncated exponential model (which admits an explicit solution) and our soft-guidance framework in Section \ref{sc:reinforce} lead to diverging queue lengths. Nevertheless, Figure~\ref{fig:conditional_queue} shows that the diffusion model with soft guidance leads to smaller divergence rates, as it captures a broader range of arrival and service times.

We observe that queue lengths in wards with restricted routing—where certain patient types can be served only by specific wards—are highly sensitive to variations in both arrival and service rates. By inspecting the simulation logs, we find that the wards serving patient types $6$ and $7$ exhibit explosive growth; accordingly, we double the number of servers assigned to these wards. The resulting performance is shown in Figure~\ref{fig:conv_compare}, where the system converges after approximately 30,000 days of simulation. As in the unstable regime, conditional diffusion generation consistently produces shorter queues than the blue benchmark, reflecting the smoothing effect induced by soft guidance. These results highlight the sensitivity of the queueing system to small distributional shifts and underscore the need for additional server capacity in critical wards during peak seasons. They also emphasize the importance of realistic simulators for stress testing, especially in unstable or near-unstable regimes.


\begin{figure}
    \centering
    \includegraphics[width=0.6\linewidth]{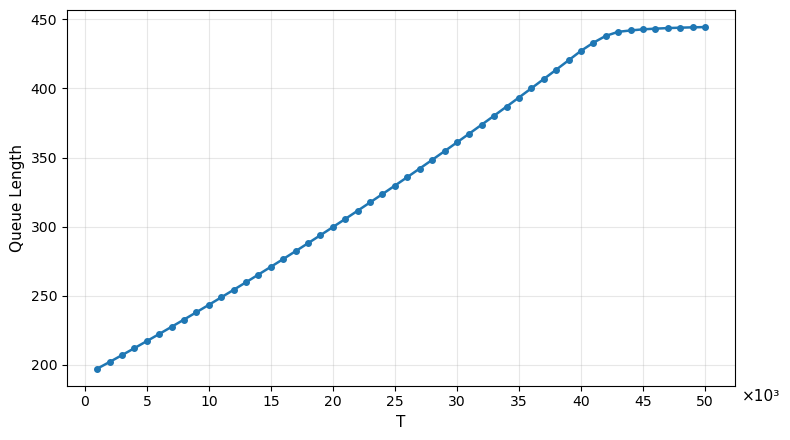}
    \caption{Limiting original total queue length.}
    \label{fig:orig_queue}
\end{figure}

\begin{figure}[H]
    \centering
    \includegraphics[width=0.6\linewidth]{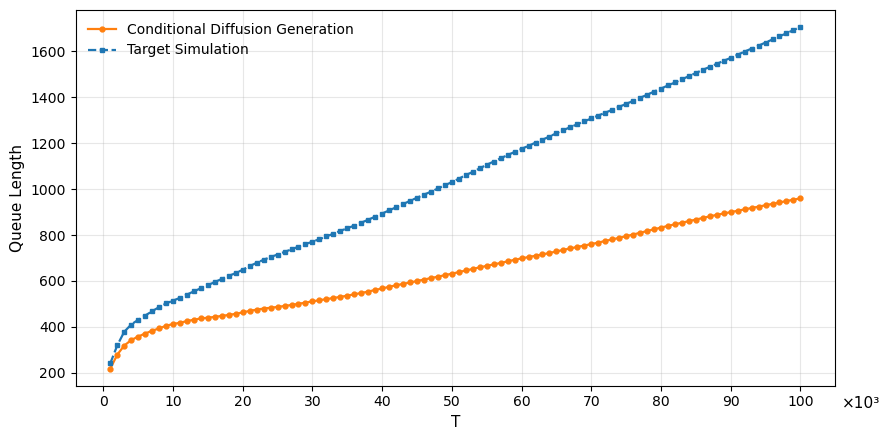}
    \caption{Comparison of queue length under the guidance of shorter patient inter-arrival times and longer service times (blue: hard truncated exponential distribution; orange: soft guidance).}
    \label{fig:conditional_queue}
\end{figure}

\begin{figure}[H]
    \centering
    \includegraphics[width=0.6\linewidth]{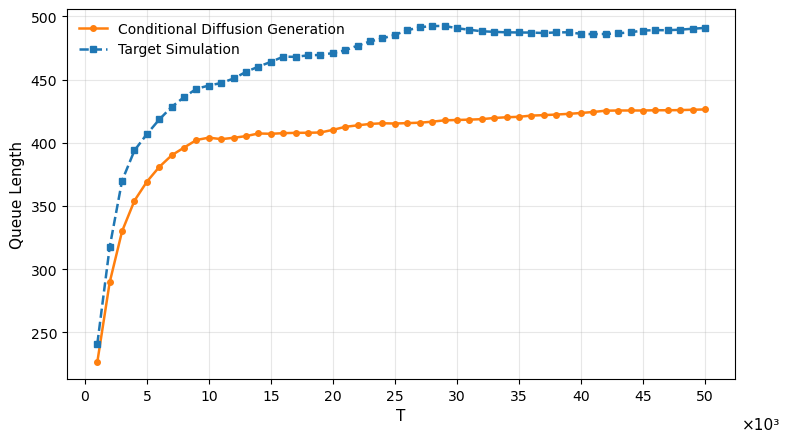}
    \caption{Comparison of queue length after adding more servers.}
    \label{fig:conv_compare}
\end{figure}



\bibliographystyle{abbrv}
\bibliography{unique}

\end{document}